\documentclass[twoside]{article}

\usepackage[accepted]{aistats2023}
\usepackage[round]{natbib}

\usepackage{amsmath,amsthm,amsfonts,bbm}

\newcommand{\defeq}{:=}

\newcommand{\makevector}[1]{{\mathbf #1}}

\newcommand{\wvec}{{\makevector{w}}}
\newcommand{\xvec}{{\makevector{x}}}

\newcommand{\zvec}{{\makevector{z}}}
\newcommand{\Ivec}{{\makevector{I}}}

\newcommand{\RR}{{\mathbb{R}}}
\newcommand{\EE}{{\mathrm{E}}}

\newcommand{\bbP}{{\mathbb{P}}}

\newcommand{\VV}{{\mathrm{V}}}

\newcommand{\calU}{{\mathcal{U}}}

\newcommand{\cS}{\mathcal{S}}
\newcommand{\cA}{\mathcal{A}}

\newcommand{\Bmat}{\mathbf{B}}

\newcommand{\Xmat}{\mathbf{X}}
\newcommand{\Ymat}{\mathbf{Y}}
\newcommand{\Smat}{\mathbf{\Sigma}}

\theoremstyle{definition} 

\newtheorem{corollary}{Corollary}

\newtheorem{lemma}{Lemma}
\newtheorem{theorem}{Theorem}
\newtheorem{definition}{Definition}

\newcommand{\inv}{{\raisebox{.2ex}{$\scriptscriptstyle-1$}}}

\usepackage[dvipsnames]{xcolor}
\usepackage[hidelinks]{hyperref} 
\usepackage{graphicx}
\usepackage{caption}
\usepackage{subcaption}
\usepackage{amssymb}
\usepackage{algorithm}
\usepackage[noend]{algorithmic}
\usepackage{float}
\usepackage{enumitem}

\newcommand{\vincent}[1]{\textcolor{blue}{[Vincent: #1]}}

\usepackage{soul}
\usepackage[colorinlistoftodos]{todonotes}
\newcommand{\yash}[1]{\todo[color=blue!40]{Yash: #1}}

\newcommand{\stateSet}{\mathcal{S}}
\newcommand{\stateRV}{S}
\newcommand{\state}{s}
\newcommand{\sample}{D}

\begin{document}

%

%

\twocolumn[

\aistatstitle{Asymptotically Unbiased Off-Policy Policy Evaluation when Reusing Old Data in Nonstationary Environments}

\aistatsauthor{Vincent Liu \And Yash Chandak \And  Philip Thomas \And Martha White }

\aistatsaddress{ University of Alberta \And  Stanford University \And University of Massachusetts Amherst  \And University of Alberta } ]

\begin{abstract}
In this work, we consider the off-policy policy evaluation problem for contextual bandits and finite horizon reinforcement learning in the nonstationary setting. Reusing old data is critical for policy evaluation, but existing estimators that reuse old data introduce large bias such that we can not obtain a valid confidence interval. Inspired from a related field called survey sampling, we introduce a variant of the doubly robust (DR) estimator, called the regression-assisted DR estimator, that can incorporate the past data without introducing a large bias. The estimator unifies several existing off-policy policy evaluation methods and improves on them with the use of auxiliary information and a regression approach. We prove that the new estimator is asymptotically unbiased, and provide a consistent variance estimator to a construct a large sample confidence interval. Finally, we empirically show that the new estimator improves estimation for the current and future policy values, and provides a tight and valid interval estimation in several nonstationary recommendation environments. 
\end{abstract}

\section{INTRODUCTION}

Off-policy policy evaluation (OPE) is the problem of estimating the expected return of a target policy from a dataset collected by a different behavior policy. OPE has been used successfully for many real world systems, such as recommendation systems \citep{li2011unbiased} and digital marketing \citep{thomas2017predictive}, to select a good policy to be deployed in the real world. A variety of estimators have been proposed, particularly based on importance sampling (IS) \citep{hammersley1964sampling} and modifications to reduce variance, such as 
self-normalization \citep{swaminathan2015self}, direct methods that use reward models and variance reduction techniques like the doubly robust (DR) estimator \citep{dudik2011doubly,jiang2016doubly,thomas2016data}. Often high-confidence estimation is key, with the goal to estimate confidence intervals around these value estimates that maintain coverage without being too loose \citep{thomas2015hcpi,thomas2015high,swaminathan2015counterfactual,kuzborskij2021confident}.  

Much less work has been done, however, for the nonstationary setting where the reward and transition dynamics 
change over time. Extending these approaches to the nonstationary setting is key as most real world systems change with time, or appear to due to partial observability. In this setting, we face a critical bias-variance tradeoff: using past data introduces bias, but not using past data introduces variance. \citet{jagerman2019people} introduced the sliding-window IS and exponential-decay IS estimator, that gradually reduces the impact of older data to control the bias-variance tradeoff. 
There is some other work predicting future OPE values for a target policy in a nonstationary environment, by using time-series forecasting \citep{thomas2017predictive,chandak2020optimizing}; the goal there, however, is to forecast future policy values using past value estimates, rather than to estimate the current value.


Much of the other work tackling nonstationary problems has been for policy optimization. There is a relatively large body of work on nonstationary bandits in the on-policy setting (e.g., see \cite{yu2009piecewise}). More pertinent to this work is a recent approach in the off-policy setting \citep{hong2021non}. Their focus, however, is on the use of change point detection and hidden Markov models for policy optimization in the online phase. As a result, these ideas do not directly apply to nonstationary OPE.

In this work, we propose a new approach for nonstationary OPE by exploiting ideas from a related field called \emph{survey sampling} \citep{cochran1977sampling}, where handling nonstationary data has been a bigger focus.
We propose a variant of the DR estimator, called the regression-assisted DR estimator, for nonstationary environments.
We exploit two ideas:
(1) using auxiliary variables from the past data to build a proxy value and incorporate the proxy value in the estimator without introducing bias, and (2) a regression approach on top of the proxy value to reduce variance further. Using the regression approach introduces some bias, however, we prove that the estimator is asymptotically unbiased and provide a consistent variance estimator to construct a large sample confidence interval (CI). 
Moreover, we show that this regression-assisted estimator unifies several existing OPE methods, including the weighted IS estimator.
We empirically show that in several recommendation problems, formalized as contextual bandits, that the new estimator improves the estimation and provides a tighter and valid CI empirically compared to the sliding-window estimators. 
We then extend the idea to finite horizon reinforcement learning, and highlight similar improvements.

\section{PROBLEM SETUP}
In this section, we describe our main problem setting: off-policy policy evaluation (OPE) in the nonstationary setting. 
To convey the core insights of our paper precisely, we first focus on contextual bandits. 

\textbf{Notation.} We start by describing the standard stationary setting for OPE in the contextual bandit setting.
Let $\stateSet$ be a set of contexts, $\cA$ be a set of actions, and 
$r:\stateSet\times\cA\to\RR$ be the reward function. The goal is to evaluate the value of a target policy $\pi$, that is, estimate 
$J(\pi)=\EE_{\stateRV\sim P, A\sim\pi(\cdot|S)}[r(\stateRV,A)]$, using an offline (off-policy) dataset. The dataset is created through the interaction of a behavior policy with the environment: (1) the environment draws a context $\state_i$ from $P\in\Delta(\stateSet)$ and (2) the behavior policy draws an action $a_i$ from $\pi_b(\cdot|\state_i)$ and observes $r_i=r(\state_i,a_i)$.
This process repeats $n$ times, giving dataset $\sample=\{(\state_i,\xvec_i,a_i,r_i)\}_{i=1}^n$. 
We assume that we also observe the context feature $\xvec_{s}\in\RR^d$ for each context $s$ in the dataset. 

\textbf{Nonstationary OPE.} Dealing with arbitrary nonstationarity may not be possible. Fortunately, many real world environments have structures than can be exploited. 
We consider a piecewise stationary setting with known change points, where the reward function changes across intervals but remains stationary within each interval. For example, an environment can be stationary within each day or each week or for a number of interactions.
We assume the set of contexts and the set of actions do not change over time. 

Let $r_k$ denote the reward function for the $k$-th interval and $D_k=\{(\state_i,a_i,r_k(\state_i,a_i))\}_{i=1}^{n_k}$ denote the data of size $n_k$ collected over the $k$-th interval. The goal is to estimate 
\begin{equation*}
    J_{k}(\pi)=\sum_{s\in\stateSet,a\in\cA}P(\state)\pi(a|\state)r_k(\state,a)
\end{equation*}
given previous datasets $D_1,\dots,D_{k-1}$ and a newly sampled $D_k$. The problem mirrors the real world where we have plenty of past data $D_1,\dots,D_{k-1}$ but only a small amount of new data $D_k$ to estimate the current value $J_{k}(\pi)$. 
We consider a stationary context distribution to present the paper succinctly, however, our methods described in the paper are applicable to the settings where the context distribution is also changing.

It is often necessary in high-stakes applications to provide confidence intervals.
Let $\mathcal D = (D_t)_{t=1}^k$ denote the set of all data collected across different intervals.
Given $\mathcal D$ and a desired level of failure probability $\alpha \in (0, 1)$, it would be ideal to estimate a high confidence lower bound $\operatorname{CI}^{-}$ and a high confidence upper bound $\operatorname{CI}^{+}$ such that
\begin{align*}
    \Pr\left(\operatorname{CI}^{-}(\mathcal D, \alpha)\leq J_k(\pi) \leq  \operatorname{CI}^{+}(\mathcal D, \alpha) \right) = 1-\alpha
\end{align*}
where the probability is under the randomness of $D_k$ and conditional on all old data $D_1,\dots,D_{k-1}$.

\section{BACKGROUND}
In this section, we review existing estimators for stationary OPE and describe how OPE can be written using the survey sampling formulation. We use this survey sampling formulation to introduce the proposed estimators in the next section. 

\subsection{Estimators for Stationary OPE}
A foundational strategy to estimate $J(\pi)$ in stationary OPE is to use importance sampling.
The IS estimator is given by
\begin{equation*}
    \hat J_\text{IS}(\pi) 
    = \frac{1}{n} \sum_{i=1}^n \frac{\pi(a_i|\state_i)}{\pi_b(a_i|\state_i)} r(\state_i,a_i).
\end{equation*}
%
This IS estimator can have high variance since the importance ratio can be very large. The weighted IS (WIS) estimator \citep{sutton1998introduction}, also known as the self-normalized estimator \citep{swaminathan2015self}, normalizes the importance weights and is more commonly used. The WIS estimator is given by
\begin{equation*}
    \hat J_\text{WIS}(\pi) = \sum_{i=1}^n \frac{\pi(a_i|\state_i)/\pi_b(a_i|\state_i)}{\sum_{j=1}^n \pi(a_j|s_j)/\pi_b(a_j|s_j)} r(\state_i,a_i).
\end{equation*}
Besides these IS-based estimators, another common estimator is the direct method (DM). We learn a reward prediction model $\hat r$ and use
\begin{equation*}
    \hat J_\text{DM}(\pi) = \frac{1}{n} \sum_{i=1}^n \sum_{a\in\cA} \pi(a|\state_i) \hat r(\state_i, a).
\end{equation*}
The doubly robust (DR) estimator \citep{dudik2011doubly} combines the DR and the IS estimator, 
\begin{align*}
    &\hat J_\text{DR}(\pi) = \frac{1}{n} \sum_{i=1}^n\\ &\left[\frac{\pi(a_i|\state_i)}{\pi_b(a_i|\state_i)} (r(\state_i,a_i) - \hat r(\state_i, a_i)) +
    \sum_{a\in\cA} \pi(a|\state_i) \hat r(\state_i, a)\right].
\end{align*}

There are other OPE estimators such as  estimators with clipping \citep{bottou2013counterfactual} or shrinkage \citep{su2020doubly}. \citet{dudik2012sample} studied the setting where the policies are non-stationary (history-dependent) but the environment is stationary, which is different from our setting. \citet{chandak2021universal} focus on estimating the reward distribution and do not discuss how to efficiently leverage past data under non-stationarity.

\subsection{OPE as Survey Sampling}


Survey sampling can be dated back to \cite{hansen1943theory,horvitz1952generalization}, where they consider the problem of selecting a sample of units from a finite population to estimate unknown population parameters.
Formally, let $\calU=\{1,\dots,N\}$ be the population of interest, $y_i$ be the study variable and $\xvec_i$ be the auxiliary variable for the unit $i\in \calU$. A subset of the population, called a sample, is selected according to a sampling design. We observe the study variable for units in the sample, and the goal is to estimate the population total of the study variables $t_y=\sum_{i\in\calU} y_i$. 

To formalize OPE under survey sampling, let the population be $\calU=\stateSet\times\cA$ and the study variable be $y_{\state,a}=P(\state) \pi(a|\state) r(\state,a)$. The population total of $y$ is the value of the policy: $t_y =\sum_{(\state,a)\in\calU} y_{\state,a} = J(\pi)$. The weighting $P(\state)\pi(a|\state)$ goes into the study variable since the goal is to estimate the total of study variable without weighting. Even though we have $P(\state)$ in the study variable, the term often cancels out in the estimator.

This formulation has some subtle differences from the standard OPE formulation. First, it assumes that $\stateSet\times\cA$ is finite, since $\calU$ is finite in survey sampling. Second, the study variable is fixed, that is, the reward function is deterministic.
These limitations can be overcome by assuming that the finite population is generated as a random sample from an infinite superpopulation; this superpopulation model is discussed in the appendix. For the main body, we assume a finite population with fixed study variables.  

Of particular interest for nonstationarity is the \emph{model-assisted} approach for survey sampling \citep{sarndal2003model}. 
The key idea is to use the auxiliary variable $x_{s,a}$ to form a proxy value $\hat y_{s,a}$ such that $\hat y_{s,a}$ is close to the study variable $y_{s,a}$. A simple example is that the auxiliary variable $x_{\state,a}$ might be the value of $y_{\state,a}$ at a past time and we can use proxy value $\hat y_{s,a}=x_{\state,a}$. A general form for a model-assisted estimator, assuming the population total of the proxy value is known, is the difference estimator \citep{cassel1976some}: $\sum_{(\state,a)\in\calU} \hat y_{\state,a} + \sum_{(\state,a) \in\sample} \frac{y_{\state,a}-\hat y_{\state,a}}{np_{\state,a}}$ where $p_{\state,a}$ is the probability of selecting the pair $(s,a)$. This estimator is unbiased 
and can be much lower variance, if the proxy value is close to the study variable. This strategy is like adding a control variate, but specific to survey sampling since the source of stochasticity is different than the typical Monte Carlo setting.   




\section{OPE ESTIMATORS UNDER NONSTATIONARITY}
\label{sec:NSOPE}


In the section, we propose an estimator for nonstationary environments. 
There are two popular strategies that consider the bias-variance tradeoff when reusing the past data in non-stationary environments: sliding window IS and exponential decay IS \citep{jagerman2019people}.
The sliding window IS estimator directly uses the IS estimator for the data in the most recent $B$ intervals.  Though not proposed in the original work, it is natural to extend this idea to other estimators. For example, for the direct method, we can build a reward model from the data in the most recent $B$ intervals and evaluate the policy with the reward prediction.  

The window size $B$ controls the bias-variance tradeoff. If $B=0$ then we only use the most recent data $D_k$: the estimator does not introduce bias by using past data but suffers high variance due to having a small sample size. If we use a large $B$, the estimator might introduce large bias but might have lower variance due to a larger sample size. 
Sliding window estimators require carefully choosing $B$ to balance the bias from using past data and the variance from not using past data. The balance usually depend on how fast the environment is changing, which is usually unknown.
Moreover, even with a small value of $B$, the bias of the sliding window estimator can be so large that the confidence interval is invalid, as we will show in the experiment section. 

Therefore, the main question that we aim to address is:

\emph{How can we reuse the past data for nonstationary OPE without introducing large bias?}

One natural way to leverage the past data would be to use the DR estimator with a reward prediction learned from the past data, as described in the following section. 
However, naively using the past data to construct a reward prediction may not be the best approach in the nonstationary setting.  
This raises a followup question:
\emph{How can we obtain a good reward prediction to both reduce the error of estimation and also obtain tight CIs?}
To address this challenge we draw inspiration from the survey sampling literature, and propose the regression-assisted DR estimator, that helps reduce variance further and provides tighter CIs.

\subsection{The Difference and DR estimator}
\label{sec:diff}

We can leverage the idea of the difference estimator in survey sampling, for our nonstationary OPE setting. We can use the past data $D_{k-B},\dots,D_{k-1}$ to build a reward prediction $\hat r_{k}$ as a function of the context feature and the action: $\hat r_{k}(s,a)=m(\xvec_{s},a;\theta)$ for some function $m$ parameterized by $\theta$. The reward prediction can be used as the proxy value in the estimator.
The resulting difference estimator, for interval $k$, is
\begin{align}
    \hat t_{\text{Diff},k} 
    &= \sum_{(\state,a)\in\calU} P(\state) \pi(a|\state)  \hat r_{k}(\state,a) \nonumber\\
    &+ \frac{1}{n} \sum_{(\state,a) \in D_k} \frac{\pi(a|\state) }{\pi_b(a|s)}(r(\state,a) - \hat r_{k}(\state,a)).
    \label{eq:diff}
\end{align}
The elegance of this approach is that we can leverage past data by incorporating it into the proxy value in the difference estimator, without introducing any bias. 

While the estimator is unbiased, the variance
depends on the quality of the reward prediction $\hat r_{k}$ \citep{thomas2016data}. 
The environment is nonstationary, so past data has to be used carefully to get a good estimate, and in some cases, the estimate may be poor.
In the next section, we discuss how to obtain a better prediction by fitting a regression on top of the reward prediction. 

A careful reader would have noticed one other nuance with the above difference estimator: it requires the population total of the proxy value $\hat y_{s,a}=P(\state) \pi(a|\state)  \hat r_{k}(\state,a)$, which is the first term in Eq \eqref{eq:diff}. In some cases, this information may be known and it should be leveraged to get a better estimator for OPE. In other cases, we will need to estimate it.
In the standard contextual bandit setting, given a sample $\sample_k$, we often assume that we know the auxiliary variable $\xvec_{s,a}$ for all units in the set $\{(\state,a):\state\in\sample_k,a\in\cA\}$. 
If we estimate the population total from $\sample_k$ with the information about the auxiliary variables, the estimator becomes
\begin{align}
    \hat t_{\text{DR},k} 
    &= \frac{1}{n}\sum_{\state\in \sample_k} \sum_{a\in\cA} \pi(a|\state) \hat r_{k}(\state,a) \nonumber\\
    &+ \frac{1}{n} \sum_{(\state,a) \in \sample_k} \frac{\pi(a|\state)}{\pi_b(a|s)} (r_{k}(\state,a) - \hat r_{k}(\state,a)).
    \label{eq:DR}
\end{align}
This estimator reduces to the DR estimator. Therefore, the DR estimator is the difference estimator when the population total of the proxy value is estimated by sample $\sample_k$. 

However, there are other options to estimate the population total, that do not result in the standard DR estimator. Of particular relevance here is that we can use past data $D'$ to estimate this population total: $\frac{1}{|\sample'|}\sum_{\state\in\sample'} \sum_{a\in\cA} \pi(a|\state)  \hat r_{k}(\state,a)$.
This term does not rely on rewards in the past data---which might not be correct due to nonstationarity---and only requires access to the auxiliary variables $\xvec_s$ in these datasets. If we assume only the rewards are nonstationary, rather than the context distribution, making these old datasets a perfectly viable option to estimate this population total. In survey sampling, this is usually motivated by assuming that there might be another survey that contains the auxiliary variables \citep{yang2020statistical}.

\subsection{The Regression-Assisted DR Estimator} \label{sec:reg}

We consider a model on top of the reward prediction from the past data to mitigate variance further.
Let $\phi_k(\state,a)^\top=(1, \hat r_{k}(\state,a))$ be the augmented feature vector with the reward prediction and define $\zvec_{\state,a}= P(\state)\pi(a|\state)\phi_k(\state,a)$. Note that $\zvec_{\state,a}$ is a function of the auxiliary variable $\xvec_{\state}$. 
We consider a (heteroscedastic) linear regression model such that the study variables $y_{s,a} \defeq P(s)\pi(a|s)r_k(s,a)$ are realized values of the random variables $Y_{s,a}$ with $\EE_\xi[Y_{s,a}] = \zvec_{s,a}^\top \beta$ and $\VV_\xi(Y_{s,a})=\sigma_{s,a}^2=P(\state)\pi(a|\state)\sigma^2$
where the expectation and variance are with respect to the model $\xi$, and $\beta,\sigma$ are the model coefficients. 
These are the assumptions underlying the regression estimator, rather than assumptions about the real world. Further, even though we consider a linear regression on the feature vector $\phi$ for the regression-assisted DR estimator, the reward prediction itself can be non-linear (e.g., a neural network).



The weighted least squares estimate of $\beta$ is $\beta_{k} = \arg\min_\beta \sum_{(\state,a)\in\calU} \frac{1}{\sigma_{\state,a}^2} \left(\zvec_{s,a}^\top\beta- y_{s,a}\right)^2$. 
Suppose the relevant matrix is invertible, $\beta_{k}$ can be estimated using sample data $\sample_k$:
\begin{align}
    \hat \beta_k
    =& \left(\sum_{(\state,a)\in D_k} \frac{\pi(a|\state)}{\pi_b(a|s)}\phi_k(\state,a) \phi_k(\state,a)^\top\right)^\inv \nonumber\\ &\left(\sum_{(\state,a)\in D_k} \frac{\pi(a|\state)}{\pi_b(a|s)}\phi_k(\state,a)r_k(\state,a)\right).
    \label{eq:betahat}
\end{align}
If we know the population total of $\zvec_{\state,a}^\top \hat\beta_k$, then the regression-assisted DR (Reg) estimator is
\begin{align}
    &\hat t_{\text{Reg},k}
    = \sum_{(\state,a)\in\calU} P(\state) \pi(a|\state) \phi_k(\state,a)^\top \hat\beta_k \nonumber\\ 
    &\ \ \ \ + \frac{1}{n} \sum_{(\state,a)\in D_k} \frac{ \pi(a|\state)}{\pi_b(a|s)} \left(r_k(\state,a)- \phi_k(\state,a)^\top\hat\beta_k\right).
    \label{eq:reg}
\end{align}
More generally, we can use the same data or the past data to estimate the population total, as described above.



This $\hat \beta_k$ consists only of the weight on the past reward model and the bias unit. This may not seem like a particularly useful addition, but because it is estimated using $D_k$, it allows us to correct the past reward prediction. 

Further, the regression-assisted DR estimator actually provides a natural way to combine existing estimators in the OPE literature, depending on the choice of the feature vector $\phi$ and the coefficients $\beta$. 
To see this, we first show how WIS can be seen as an instance of this estimator.\footnote{\cite{mahmood2014weighted} have a similar observation that the solution $\hat\beta$ is the WIS estimator if $\phi(\state,a)=1$ for all $(s,a)$. They also extend the estimate with linear features $\phi$ and use $\phi(s,a)\hat\beta$ directly, which is more related to the model-based approach. In our work, the model prediction is used as the proxy value so the resulting estimators are different.}

We provide the theoretical result in the stationary setting where $\phi$ is fixed, so we can drop the subscript $k$ for simplicity. 
For the nonstationary setting, the inference for $\hat t_{Reg,k}$ is conducted \emph{conditional on} the past data  $D_1,\dots,D_{k-1}$, so $\phi$ is again fixed and all results extends to the nonstationary setting. The proofs can be found in Appendix \ref{appendix:proof}.

\begin{theorem}[WIS as a special case of the regression-assisted estimator]
    \label{thm1}
    Suppose we use a linear regression model with univariate feature $\phi(s,a)=1$.
    Then the regression-assisted DR estimator with estimated coefficient $\hat\beta$ from Eq \eqref{eq:betahat} has the same form as the WIS estimator:
    \begin{align}
        \hat t_\text{Reg} 
        = \sum_{(\state,a)\in\sample}\frac{\pi(a|\state)/\pi_b(a|s)}{\sum_{(s',a')\in S}\pi(a'|s')/\pi_b(a'|s')}r(s,a).
        \label{eq:WIS}
\end{align}
\end{theorem}


The result provides a novel perspective for the WIS estimator: it can be viewed as fitting a regression to predict the reward with a constant feature. As a result, the only difference between the regression-assisted DR estimator and the WIS estimator is the choice of feature vector for reward prediction. If there are other features that might be useful for predicting the reward, we can include it with the regression approach and potentially improve the WIS estimator.


In Table \ref{tab1}, we show that we can recover other estimators based on different choices for the coefficients $\beta=(\beta_1,\beta_2)^\top$ with the feature vector $\phi(\state,a)^\top=(1, \hat r(\state,a))$. If $\beta_1=0,\beta_2=0$, we recover the IS estimator. If $\beta_1=0,\beta_2=1$, we recover the difference estimator or the DR estimator. If $\beta_2=0$ and $\beta_1$ is learned from data, we recover the WIS estimator. 
\begin{table}[h]
    \centering
    \caption{A Unifying View of Existing Estimators.}
    \label{tab1}
    \begin{tabular}{c|cccc}
        & IS & DR & WIS & Reg  \\ \hline
        $\beta_1$ & $0$ & $0$ & $\hat \beta_1$ & $\hat \beta_1$ \\
        $\beta_2$ & $0$ & $1$ & 0 & $\hat \beta_2$\\
    \end{tabular}
\end{table}

There are other approaches to estimate the coefficients from data. 
The more robust DR estimator \citep{farajtabar2018more} minimizes the estimated variance $\hat \VV (\hat t_\text{Reg})$ with respect to the coefficient to achieve the lowest asymptotic variance among all coefficient $\beta$ under some mild conditions. \cite{kallus2019intrinsically} further consider an expanded model class on top of the reward prediction and minimize the estimated variance among both the expanded model class and the reward prediction model class. 
However, it often unclear how large the sample size needs to be such that the estimator with the lowest asymptotic variance indeed has a lower variance against other estimators in practice. 
On the other hand, there is a considerable literature in survey sampling on improving estimation for the total and variance estimator when the sample size is small or the feature vector is high dimensional. For example, \cite{breidt2000local,mcconville2017model}
propose different models as an alternative to the linear regression model. These regression models can be potentially more useful for feature selection or to find a model that fits the population well.





\section{THEORETICAL ANALYSIS}
\newcommand{\MSE}{\mathrm{MSE}}
\newcommand{\AMSE}{\mathrm{AMSE}}
\newcommand{\Cov}{\mathrm{Cov}}
\newcommand{\AV}{\mathrm{AV}}
\newcommand{\V}{\mathrm{V}}

In the regression approach, if the coefficients are estimated from the same data $D_k$, the estimator becomes biased. For example, the DR estimator is unbiased since the coefficients are fixed, and the WIS estimator is biased since one of the coefficients is estimated.
In this section, we show that even if we run the regression on the same data we use to build the estimator, the regression-assisted DR estimator still enjoys asymptotic properties.


To prove these theoretical properties, there are a number of results from the survey sampling literature that we build on. For completeness, we provide a brief overview of survey sampling in Appendix \ref{sec:ss}, and the proof of these properties under survey sampling notation in Appendix \ref{appendix:proof}. 

\begin{theorem}[Properties of the estimator]
    \label{thm2}
    Let $\AV(\cdot)$ denote the asymptotic variance in term of the first order, that is $\V(\cdot)=\AV(\cdot)+o(n^\inv)$, we have
    (1) $\hat t_\text{Reg}$ is asymptotically unbiased with a bias of order $O(n^{-1})$, and (2) 
    \begin{align*}
        &\AV(\hat t_\text{Reg}) 
        = \frac{1}{n}\\
        &\left(\sum_{(\state,a)\in U} P(\state) \frac{\pi(a|\state)^2}{\pi_b(a|s)}(r(\state,a) - \phi(\state,a)^\top \beta)^2 - t_{e}^2\right)
    \end{align*}
    where $t_e = \sum_\calU P(s)\pi(a|s)(r(\state,a) - \phi(\state,a)^\top \beta)$.
\end{theorem}

\paragraph{Variance estimation for the regression-assisted DR estimator.}
The exact form of the variance of the regression-assisted DR estimator is often difficult to obtain, so we use the approximate variance from Theorem \ref{thm2}. Replacing the unknown $\beta$ by the sample-based estimate $\hat\beta$, we have a variance estimator
\begin{align*}
    &\hat \V (\hat t_\text{Reg})
    = \frac{1}{n(n-1)}\\
    &\left[\sum_{(\state,a)\in \sample} \left(\frac{\pi(a|\state)}{\pi_b(a|\state)} (r(\state,a) - \phi(\state,a)^\top\hat\beta)\right)^2 - n \hat t_e^2 \right]
\end{align*}
where $\hat t_e = \sum_{\sample} \frac{\pi(a|\state)}{n\pi_b(a|s)}(r(\state,a) - \phi(\state,a)^\top\hat\beta)$. 
\cite{sarndal1989weighted} propose the weighted residual technique which can potentially result in better interval estimation. See Appendix \ref{appendix:regression_variance} for a derivation. 


Finally, we show the variance estimator is consistent and the regression-assisted estimator is asymptotically normal. 
\begin{theorem}
    \label{thm3}
    The variance estimator $\hat \VV(\hat t_\text{Reg})$ is consistent, and $\frac{\hat t_\text{Reg} - t_y}{\sqrt{\hat \VV(\hat t_\text{Reg})}} \overset{D}\to \mathcal{N}(0,1)$.
\end{theorem}

Based on Theorem \ref{thm3}, we can construct a large sample CI. 
\begin{corollary}
    Let $\hat\sigma \!=\! \sqrt{\hat \VV(\hat t_\text{Reg})}$ and $z_{\alpha}$ denote the $100(1-\alpha)$ percentile of the standard normal distribution, then
    \begin{align*}
        \Pr\left(\hat t_\text{Reg} - z_{\alpha/2} \hat\sigma \leq J(\pi) \leq \hat t_\text{Reg} + z_{\alpha/2} \hat\sigma \right)
        \to 1-\alpha.
    \end{align*}
\end{corollary}


\section{EXPERIMENTS}

In this section, we demonstrate the effectiveness of the regression-assisted DR estimators in a semi-synthetic and a real world recommendation environment. We compare the proposed estimator to existing estimators, including IS, WIS, DM and Diff (which is DR without estimating the population total). 
We also include the IS, WIS, DM with the sliding window (SW) approach of window size $B$. When $B=0$, SW-IS and SW-WIS is the standard IS and WIS. For Diff, Reg, we use the past data $D_{k-B},\dots,D_{k-1}$ to learn a reward prediction. 


For the semi-synthetic dataset, we follow the experimental design from \cite{dudik2011doubly}. We use the supervised-to-bandit conversion to construct a partially labeled dataset from the YouTube dataset in the \href{https://www.csie.ntu.edu.tw/~cjlin/libsvmtools/datasets/multilabel.html}{LibSVM repository}.
We construct a nonstationary environment by generating a sequence of reward functions based on the environment design in \cite{chandak2020optimizing}. For each true positive context-action pair in the original classification dataset, the reward follows a sine wave with some noise over time. We use PCA to reduce the dimension of the context features to 32. The target policy is obtained by training a classifier on a small subset of the original classification dataset. 

We adapt the Movielens25m dataset \citep{harper2015movielens} for the real world experiment. To construct a nonstationary environment, we divide the rating data chronologically. Each interval contains the rating data for 60 days and we use the rating data for $K=24$ intervals ending November 21, 2019. We consider only active users who gave at least one rating for at least half of the $K$ intervals, resulting in a total number of around 2000 users. During each interval, we compute the rating matrix $r(u,g)$ for each user and genre by averaging the user $u$'s rating for all rated movies in the genre $g$. As a result, we have a sequence of rating functions which represent users' average rating for each genre over time. The user features are built by matrix factorization on the average rating data with hidden size $32$, and the target policy is obtained by training a classifier on a small subset of the average rating data. 

For the OPE objective, we consider an uniform weighting $P(\state) = 1/|\stateSet|$ for all users $\state$. We also let $n_k=\alpha|\stateSet|$ for all $k$ and $\alpha \in \{0.1, 1.0\}$. 
For each interval $k=0,\dots,K$, we sample data $D_k$ using a random policy.
For estimators that require a reward prediction, we build the reward prediction by linear regression on historical data for each action separately, which is the same approach used in \cite{dudik2011doubly}. More experiment details can be found in Appendix \ref{appendix:exp_details}, and Algorithm \ref{algo:exp_code} describes the experimental procedure.

In nonstationary OPE, the aim is to estimate $J_k(\pi)$ with data from $D_1$ to $D_k$. All of the estimators discussed in this paper, however, can be extended to predict the future values using the ideas from \citet{thomas2017predictive,chandak2020optimizing}. Suppose we have the OPE estimators for each interval up to interval $k$, that is, $\hat J_1(\pi), \dots,\hat J_k(\pi)$, we can fit these data to a forecasting model to predict the future value $J_{k+1}(\pi),\dots,J_{k+\delta}(\pi)$. 
We therefore test both settings: estimating $J_k(\pi)$ and predicting $J_{k+1}(\pi)$. For the experiments predicting $J_{k+1}(\pi)$, we adapt the method proposed in \cite{chandak2020optimizing} and predict the future values by fitting a regression.
That is, $\hat J_{k+\delta}(\pi)=\psi(k+\delta)^\top \hat\wvec_k$ where $\hat\wvec_k$ is the OLS estimator for the regression problem with feature map $\psi(t)=(cos(2\pi t n))_{n=0}^{d-1}$ and target $\hat J_t(\pi)$ for $t=1,\dots,k$, where we set $d=5$ in the experiment. 

\textbf{Sensitivity to window size and sample size.}
We vary the window size $B$ and sample size $n$, and report the sensitivity plot in Figure \ref{fig:sensitivity}.
The error is averaged over $K$ intervals, that is, $\mathrm{RMSE}=\sqrt{\frac{1}{K}\sum_{k=1}^K (\hat J_k(\pi) - J_k(\pi))^2}$.  
We can see that the sliding window (SW) estimators, including SW-IS, SW-WIS and SW-DM, are sensitive to the window size, while Diff and Reg are robust to the window size. 
Reg outperforms IS and WIS and simply using $B=1$ reduces the error by a large margin. 
Reg also has a lower error compared to Diff, especially with small window size and sample size. This suggests that Reg is more robust to a bad reward prediction from the past data, which implies it is more robust to the speed of the nonstationarity. 

\begin{figure*}[h]
    \centering
    \centering
    \includegraphics[width=0.84\textwidth]{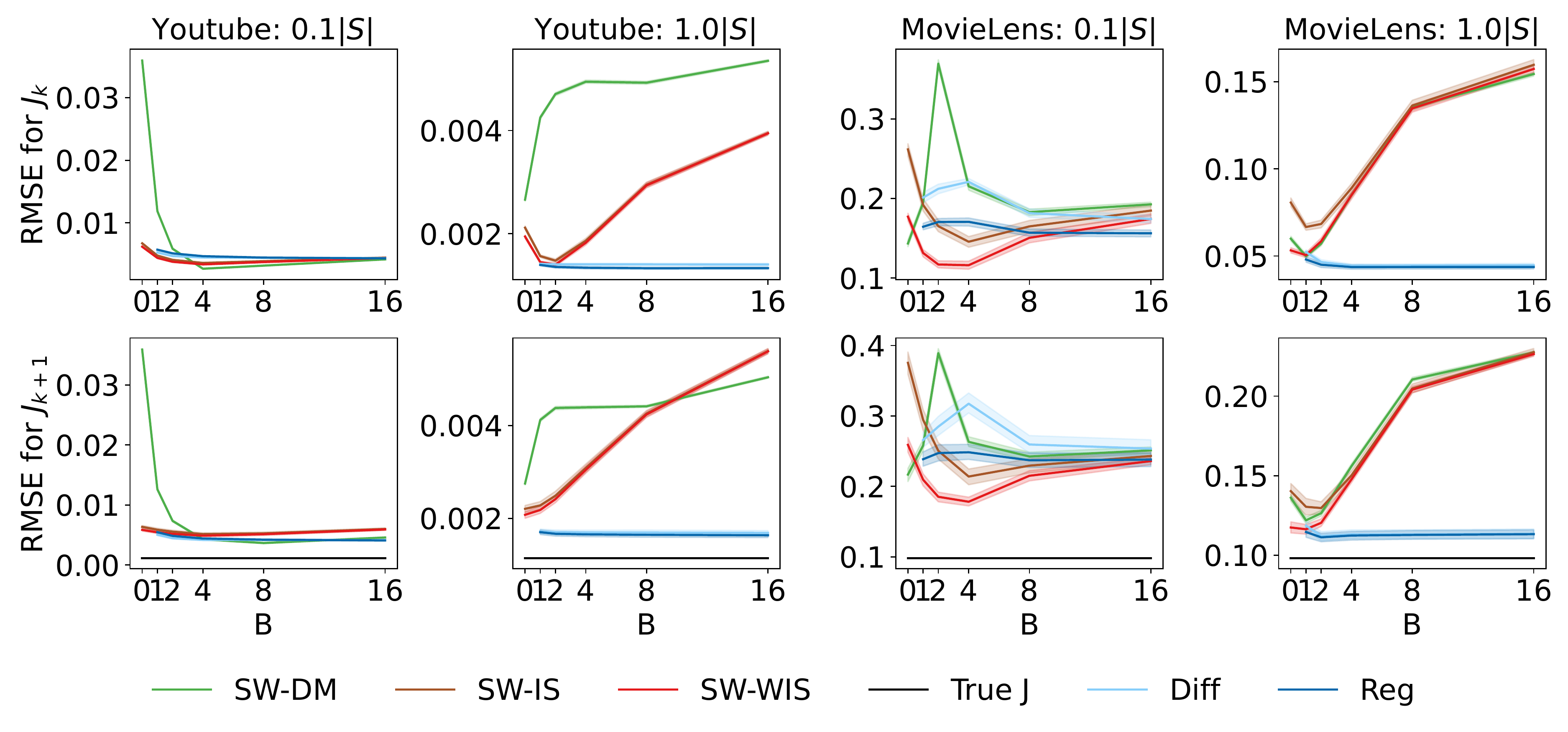}
    \caption{Sensitivity curves. \textbf{Top row: estimating $J_k(\pi)$.} \textbf{Bottom row: predicting $J_{k+1}(\pi)$.} ``True $J$'' is the baseline if we use the true values to predict the future values.
    The number are averaged over 30 runs with one standard error. Across runs, the target policy and the sequence of reward functions are fixed, but the sampled data is random. }
    \label{fig:sensitivity}
\end{figure*}

We report the error for predicting the future value $J_{k+1}(\pi)$ in Figure \ref{fig:sensitivity}. 
Reg has the lowest error for predicting the future values except in MovieLens with small sample size. 
We also find that even SW-DM and SW-IS have low error for estimating the current value $J_k$ for some hyperparameters, they still have high error for predicting the future value $J_{k+1}$. 
We hypothesize that approximately unbiased estimators generally have better future prediction even though they might have high variance. It is possible that the forecasting model cancels out the noise in approximately unbiased estimators and results in better future value prediction.

\begin{figure*}[h]
    \centering
    \centering
    \includegraphics[width=0.84\textwidth]{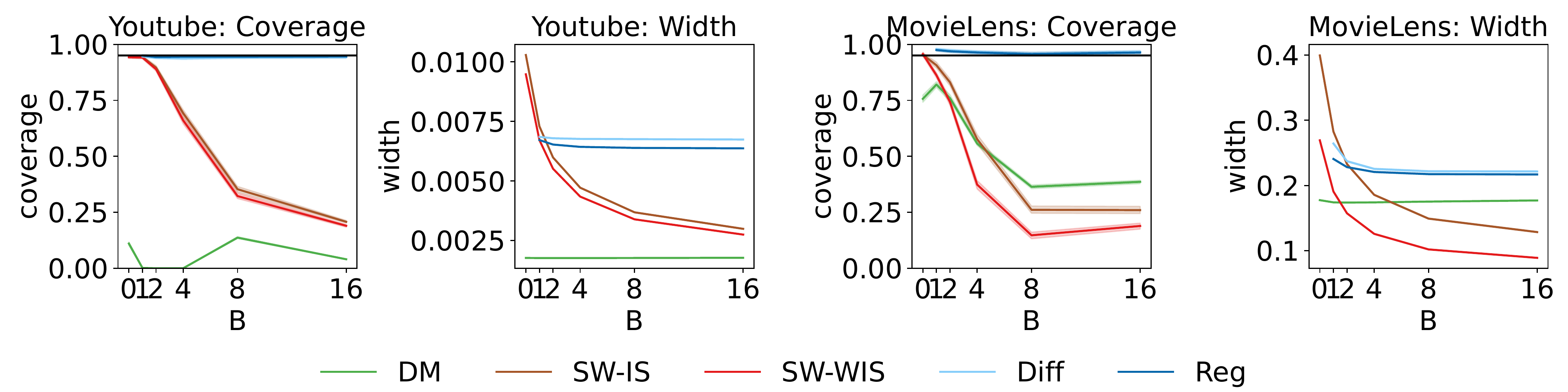}
    \caption{The empirical coverage and the width of CIs. Higher coverage and lower width is better.}
    \label{fig:coverage}
\end{figure*}

\textbf{Empirical validation of the interval estimation.}
We use $\hat J_k(\pi) \pm 1.96 \sqrt{\hat \VV(\hat J_k(\pi))}$ as the approximate $95\%$ CI. 
We report the empirical coverage of the CI using the estimated variance in Figure \ref{fig:coverage}. The empirical coverage is defined as the number of rounds such that the CI contains the true value divided by the total number of rounds $K$.
The results shown here are with $n=1.0|\stateSet|$. 
IS ($B = 0$), WIS ($B = 0$), Diff and Reg all have the desired coverage and Reg has the smallest width. All sliding window estimators have large bias when $B>0$, so the coverage is poor and it is unclear how to compare to estimators with the desired coverage.
Note that even with a small value of $B$, for example, $B=1$ in MovieLens, sliding window estimators fail to provide a valid CI. 
The result suggests that Reg provides an accurate and tight CI.

\textbf{Empirical investigation of the feature vectors.}
Besides using one past reward prediction as the only feature, we also investigate the utility when we (1) include the context features; and (2) include separate past reward predictions, that is, $\phi_k(s,a)=(1,\hat r_{k-B}(s,a),\dots,\hat r_{k-1}(s,a))$ where we learn a reward model $\hat r_t$ for interval $t$ from data $D_t$ separately. 
Since these additional features could be correlated, we use ridge regression when estimating the coefficients. The regularization parameter is chosen by cross-validation. 

We aim to answer two questions: (1) whether including the context feature or the past reward predictions helps, and (2) how we should include the past reward information. 
To answer the questions, we test five different feature vectors: (a) Reg: we use one past reward prediction as described in Section \ref{sec:reg}, with and without the context features, (b) Reg-AR: we use separate past reward predictions with and without context features, and (c) Reg-Feature: we use context features only. 
We show the comparison in Figure \ref{fig:multivariate}. 
We find that including the context features helps in general, however, using only the context features is not sufficient. 
The past reward information helps deal with nonstationarity.
Moreover, using separate predictions only improves the accuracy in MovieLens with $n=0.1|\cS|$.
In these experiments, there was no one dominant way to include past reward information, and more experimentation is needed to understand when one might be preferred.


\begin{figure*}[t]
    \centering
    \includegraphics[width=0.82\textwidth]{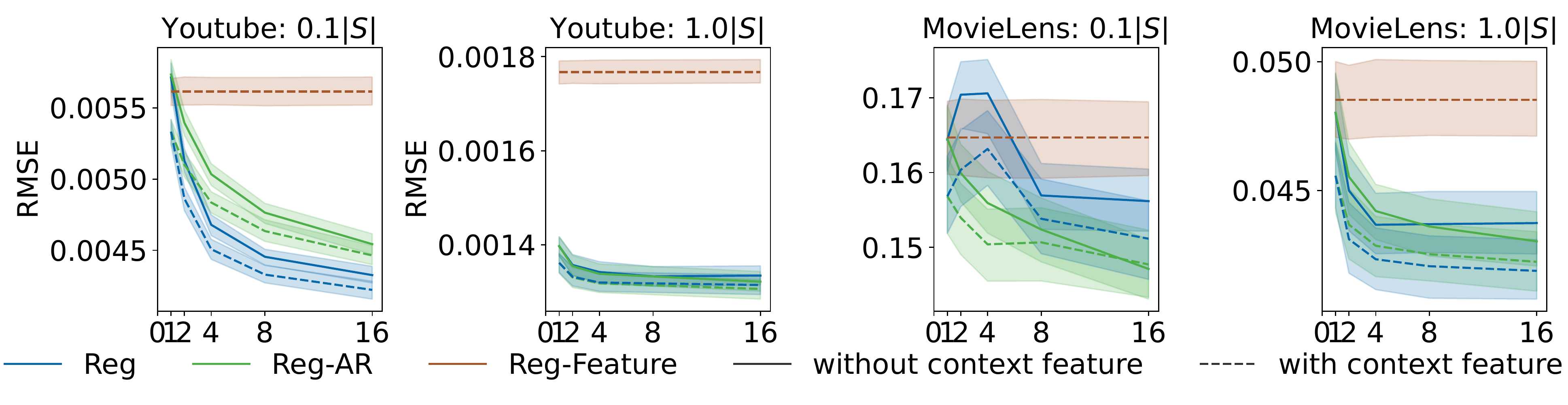}
    \caption{Comparison of different feature vectors for estimating $J_k(\pi)$. }
    \label{fig:multivariate}
\end{figure*}

We provide an ablation study to investigate the impact when the population total of the proxy value is estimated in Appendix \ref{appendix:more_exp}.
We found that using the past data to estimate the population total results in very similar performance as we know the population total. 


\begin{figure*}[!ht]
    \centering
    \centering
    \includegraphics[width=0.82\textwidth]{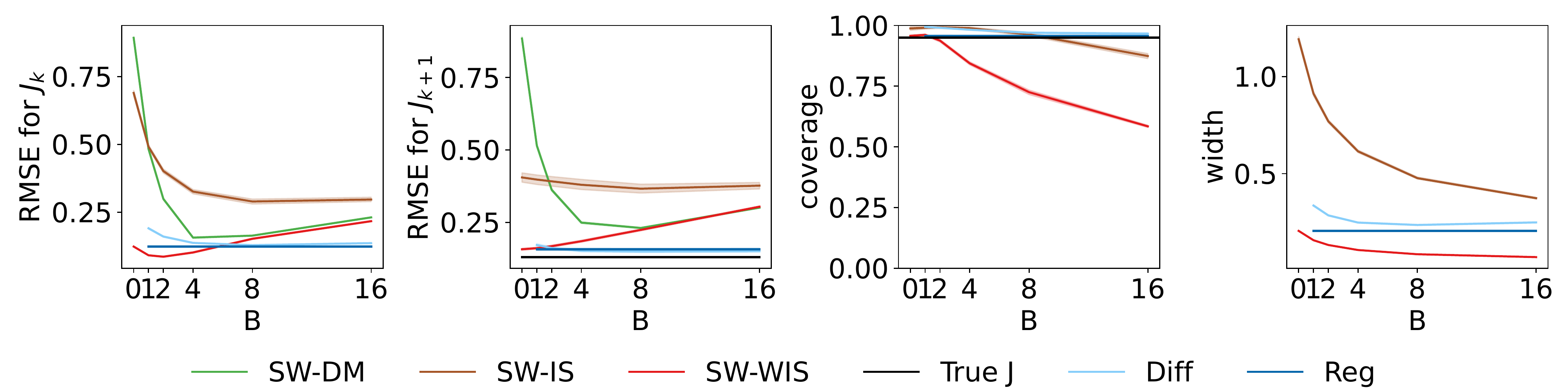}
    \caption{Results for the RL environment. \textbf{First column: estimating $J_k(\pi)$. Second column: predicting $J_{k+1}(\pi)$. Third and fourth column: coverage and width of CI}.}
    \label{fig:rl}
\end{figure*}

\section{EXTENSION TO REINFORCEMENT LEARNING}
The estimators for contextual bandits can be extended to finite horizon reinforcement learning (RL). Let $M=(\stateSet,\cA,P,r,H,\nu)$ be a finite horizon finite MDP. Our goal is to estimate the value of a policy $J(\pi)=\sum_{\tau\in(\stateSet\times\cA)^H} \bbP^\pi_M(\tau) R(\tau)$ where $\bbP^\pi_M(\tau)=\nu(s_0)\pi(a_0|s_0)P(s_1|s_0,a_0) \dots \pi(a_{H-1}|s_{H-1})$ is the probability of seeing the trajectory $\tau=(s_0,a_0,\dots,s_{H-1},a_{H-1})$ by running $\pi$ in $M$, and $R(\tau) = \sum_{h=0}^{H-1} r(s_h,a_h)$.

To formalize OPE for RL under survey sampling, let $\calU = (\stateSet\times\cA)^H$ be the population containing all trajectories and $y_{\tau} = \bbP^\pi_M(\tau)R(\tau)$ be the study variable.
Note that there are many ways to view OPE for RL in the survey sampling framework, which corresponds to different existing estimators for RL such as the trajectory-wise IS, per-decision IS (PDIS) estimator and marginalized IS estimator. We provide more discussion in Appendix \ref{appendix:rl}.

\textbf{The regression-assisted estimator with fitted Q evaluation.}
We use fitted Q evaluation (FQE), which has been shown to be effective for several stationary OPE benchmarks empirically \citep{voloshin2021empirical}, to build a proxy value $\hat R(\tau)$ for each trajectories $\tau\in\calU$.

In nonstationary environments, FQE outputs $\hat Q_{k-1}(\state,a)$ from the past data $D_{k-B},\dots,D_{k-1}$, and we use $\hat R(\tau) = \hat V_{k-1}(s_0) = \sum_{a\in\cA} \pi(a|\state_0) \hat Q_{k-1}(\state_0,a)$ as the proxy value where $\state_0$ is the initial state of the trajectory $\tau$. Similar to the estimator for contextual bandits, we first estimate the coefficient with the feature vector $\phi(s_0)^\top = (1, \hat V_{k-1}(s_0))$ and use the regression-assisted DR estimator
\begin{align*}
    &\hat t_{Reg\text{-}FQE, k}
    = \sum_{s_0\in\stateSet} \nu(s_0) \phi(s_0)^\top\hat\beta_k +\\ &\frac{1}{n} \sum_{\tau \in D_k} \frac{\pi(a_0|s_0)\dots\pi(a_{H-1}|s_{H-1})}{\pi_b(a_0|s_0)\dots\pi_b(a_{H-1}|s_{H-1})}(R(\tau) - \phi(s_0)^\top\hat\beta_k).
\end{align*}
When $\nu$ is unknown, we can estimate the population total of the proxy value by $1/|\sample'|\sum_{s_0\in \sample'} \phi(s_0)^\top\hat\beta$ from the past data $\sample'$ or the same data $\sample_k$. The regression-assisted DR estimator can be viewed as a biased-corrected FQE estimator for nonstationary environments. 

\textbf{Experimental results.}
We consider an RL environment with a binary tree structure, that is, a finite horizon MDP with $H=10$, $|\cA|=2$, $|\stateSet|=|\cA|^H-1$, and an initial state $s_0$. For each state, taking action 1 leads to the left child and taking action 2 leads to the right child. The reward for each state-action pair follows a sine wave with different frequency and amplitude. The environment mimics the session-aware recommendation problem where we take a sequence of actions for one customer during a short session. We use a random policy to collect 10 trajectories for every interval. The target policy is a trained policy using Q-learning on the underlying environment.

From Figure \ref{fig:rl}, Reg has the lowest error for estimating the current value and predicting the future value in general. We show the coverage of the one-sided CI since we mainly care about the lower bound on the policy value for safe policy improvement. The results show that Reg again provides a valid and tight interval estimation, and is promising for safe policy improvement in nonstationary RL environments. 


\section{CONCLUSION}
We proposed the regression-assisted DR estimator for OPE in the nonstationary setting, inspired by estimators from the survey sampling literature. The estimator incorporates past data into a proxy value without introducing large bias, and uses a regression approach to build a reward prediction well suited for nonstationary environments. 
As far as we know, these two ideas have not been applied to nonstationary OPE.
We theoretically show that we can construct a large sample confidence interval and empirically demonstrate that the proposed estimator provides tight and valid high-confidence estimation in several recommendation environments in contextual bandits and finite horizon RL.


\bibliographystyle{plainnat}
\bibliography{main}

\newpage
\appendix
\onecolumn


\section{OVERVIEW OF SURVEY SAMPLING}
\label{sec:ss}
In this section, we introduce the survey sampling terminology and how to use it for OPE. 
Survey sampling can be dated back to \cite{hansen1943theory,horvitz1952generalization}, where they consider the problem of selecting a sample of units from a finite population to estimate unknown population parameters.
For example, if the goal is to estimate the customer satisfaction rate for a product, survey sampling is concerned with selecting a subset of customers to conduct surveys.
Since then, the field has investigated a variety of practical scenarios, including dealing with missing data, handling non-stationarity and understanding to how to leverage auxiliary information.

Formally, let $\calU=\{1,\dots,N\}$ be the population of interest, $y_i$ be the study variable and $\xvec_i$ be the auxiliary variable for the unit $i\in \calU$. Continuing the above example, the population could be all customers, the study variable could be the satisfaction level, and the auxiliary variable could be the information about the customer. A subset of the population, called a sample, is selected according to a sampling design. We observe the study variable for units in the sample, and the goal is to estimate the population total of the study variables $t_y=\sum_{i\in\calU} y_i$. 

A sampling design $\Ivec=(I_1,\dots,I_N)$ is a random vector describing how the sample is drawn from the population: $I_i > 0$ means that the unit $i$ is selected in the sample and $I_i = 0$ means the unit is not selected. For example, a multinomial design is a with-replacement and fixed-size design where we draw $n$ units independently and identically according to probability $p_i$ with $\sum_{i\in\calU} p_i=1$. In this case, the design vector $\Ivec$ follows the multinomial distribution with parameters $n$ and $(p_i)_{i=1}^N$, that is, $P(\Ivec_1=i_i,\dots,\Ivec_N=i_N) = \frac{n!}{\Pi_{i=1}^N i_i!} p_i^{i_i} \dots p_N^{i_N}$ if $\sum_i i_i = n$ and $0$ otherwise. In survey sampling, the study variable is fixed and the randomness comes from the sampling design $\Ivec$. 

Given a sample $\sample$ of fixed size $n$, the Hansen-Hurwitz (HH) estimator \citep{hansen1943theory} for multinomial design is $\hat t_\text{HH} = \sum_{i\in\sample} \frac{y_i}{\EE[I_i]} = \sum_{i\in \sample} \frac{y_i}{np_i}$. The estimator $\hat t_\text{HH}$ is an unbiased estimator for $t_y$ if $p_i > 0$ for all $i\in\calU$.

This formalize OPE under survey sampling, let the population be $\calU=\stateSet\times\cA$ and the study variable be $y_{\state,a}=P(\state) \pi(a|\state) r(\state,a)$. The population total of $y$ is the value of the policy: $t_y =\sum_{(\state,a)\in\calU} y_{\state,a} = J(\pi)$. The weighting $P(\state)\pi(a|\state)$ goes into the study variable since the goal is to estimate the total of study variable without weighting. Even though we have $P(\state)$ in the study variable, the term often cancels out as we will see for the HH estimator. 


For OPE, the sampling design is the multinomial design with sampling probability $p_{\state,a}=P(\state)\pi(a|\state)$. Given a sample $\sample=\{(\state_i,a_i,r(\state_i,a_i))\}_{i=1}^n$, the HH estimator is
\begin{align*}
    \hat t_\text{HH} = \sum_{(\state,a)\in\sample} \frac{y_{\state,a}}{np_{\state,a}} = \sum_{(\state,a)\in\sample} \frac{P(\state) \pi(a|\state) r(\state,a)}{n P(\state) \pi_b(a|s)} = \frac{1}{n} \sum_{(\state,a)\in\sample} \frac{\pi(a|\state)}{\pi_b(a|s)} r(\state,a). 
\end{align*}
It has the same form as the IS estimator. In the case where the sampling design $p$ is not known and needs to be estimated by a propensity model, it is called the inverse propensity score (IPS) estimator. 

The HH estimator is called the \emph{design-based} estimator in the survey sampling literature. This is because the primary source of randomness is from the sampling design. Another approach is called the \emph{model-based} approach which assumes the study variables are generated by a superpopulation model. The goal is to model the relationship between the study variable and the auxiliary variable. The resulting estimator is similar to the direct method.


\textbf{The model-based approach.}
\label{appendix:model-based}
The model-based approach is a popular approach in the survey sampling literature. \cite{chambers2012introduction} provide an introduction for the model-based approach.
Different from design-based and model-assisted approaches, the study variables are assumed to be generated by a superpopulation model and typically depends on the auxiliary variable. 
More previously, we assume the values $y_i,\dots,y_n$ are realization of random variables $Y_1,\dots,Y_n$. The joint distribution of $Y_1,\dots,Y_n$ is denoted by $\xi$, which is called the superpopulation distribution. For example, we assume $\EE_\xi[Y_i | \xvec_i] = \xvec_i^\top \beta$ and $\VV_\xi(Y_i|\xvec_i) = \sigma_i^2$ for some unknown model parameter $\beta$ and $\sigma_i$.
The selected sample $\sample$ is treated as a constant and the sample values of $y_i$ are random. Estimation and inference are deduced conditional on the selected sample and the model. 


For OPE, we have $y_{\state,a}=P(\state)\pi(a|\state)r(\state,a)$ and auxiliary vector $\xvec_{\state,a}=P(\state)\pi(a|\state)\phi(\state,a)$. We assume a linear model: $\EE_\xi[Y_{\state,a}|\xvec_{\state,a}] = \xvec_{\state,a}^\top \beta$, $\VV_\xi(Y_{\state,a}|\xvec_{\state,a}) = \sigma_{\state,a}^2 = (P(\state)\pi(a|\state)\sigma)^2$, and $Y_{\state,a}$'s are independent.
Using the WLS estimator to estimate $\beta$
\begin{align*}
    \hat \beta 
    = \left(\sum_{(\state,a)\in\sample} \frac{\xvec_{\state,a} \xvec_{\state,a}^\top}{\sigma_{\state,a}^2}\right)^\dagger \left(\sum_{(\state,a)\in\sample} \frac{\xvec_{\state,a} y_{\state,a}}{\sigma_{\state,a}^2}\right)
    = \left(\sum_{(\state,a)\in\sample} \phi(\state,a) \phi(\state,a)^\top\right)^\dagger \left(\sum_{(\state,a)\in\sample} \phi(\state,a) r(\state,a)\right),
\end{align*}
we have the model-based estimator
\begin{align*}
    \hat t_\text{MB} = \sum_{(\state,a)\in\sample} y_{\state,a} + \sum_{(\state,a)\not\in\sample} \xvec_{\state,a}^\top \hat \beta. 
\end{align*}
That is, the population total is estimated by the total of study variables in the sample and the total of the study variables of units not in the sample. 

The model-based estimator is similar to the direct method (DM) in OPE. The key difference is that DM does not use the sample value of $y_{\state,a}$ but uses the prediction for all units, that is,
\begin{align*}
    \hat t_\text{DM} = \sum_{(\state,a)\in\calU} \xvec_{\state,a}^\top \hat \beta.
\end{align*}

The model-based survey sampling framework provide a way to do inference for the DM estimator, which is conditional on the selected sample and the model $\xi$. Let $t_\xvec = \sum_{(\state,a)\in\calU} \xvec_{\state,a}$, then
\begin{align*}
    \VV_\xi(\hat t_\text{DM}) = \VV_\xi\left(\sum_{(\state,a)\in\calU} \xvec_{\state,a}^\top \hat \beta\right) = t_\xvec^\top \VV_\xi(\hat \beta) t_\xvec = \sigma^2 t_\xvec^\top \left(\sum_{(\state,a)\in\sample} \phi(\state,a) \phi(\state,a)^\top\right)^\dagger t_\xvec.
\end{align*}
Plugging in the estimator $\hat \sigma^2 = \frac{1}{n-p} \sum_{(\state,a)\in\sample} (r(\state,a) - \phi(\state,a)^\top \hat\beta)^2$ for $\sigma$, we have an estimated variance 
\begin{align*}
    \hat \VV(\hat t_\text{DM}) = \hat \sigma^2 t_\xvec^\top \left(\sum_{(\state,a)\in\sample} \phi(\state,a) \phi(\state,a)^\top\right)^\dagger t_\xvec.
\end{align*}


\section{TECHNICAL DETAILS}
\label{appendix:proof}
For the theoretical analysis, we make the following assumptions: 
\begin{enumerate}
    \item $\forall(s,a)\in\calU$, $L_p \leq p_{s,a}$ for some real number $L_p>0$.
    \item $\forall(s,a)\in\calU$, $L_y \leq y_{s.a}\leq U_y$ for some real number $L_y,U_y$.
    \item $\forall(s,a)\in\calU$, $L_x \leq \phi(s,a) \leq U_x$ for some real vector $L_x,U_x$. The inequality holds element-wise. 
    \item The estimated matrix of the covariates $\sum_{(\state,a)\in D} \frac{\pi(a|\state)}{\pi_b(a|s)}\phi(\state,a) \phi(\state,a)^\top$ and the finite population matrix $\sum_{(\state,a)\in\calU} \phi(\state,a) \phi(\state,a)^\top$ are invertible. 
\end{enumerate}

In short, we need to make sure the data collection policy chooses each action with a non-zero probability, and the reward and feature vector are bounded.

\subsection{Proof of Theorem \ref{thm1}}
\label{appendix:proof1}

\begin{definition}[The ratio estimator]
    Let $z_{\state,a}\in\RR$ be the auxiliary variable, $t_z$ be the populating total of the auxiliary variable, which is assumed to be known, and $\hat t_\text{HH}$ and $\hat t_z$ be the HH estimator for $t_y$ and $t_z$ respectively.
    The ratio estimator is given by
    \begin{align*}
        \hat t_\text{Ratio} = t_z \frac{\hat t_\text{HH}}{\hat t_z}.
    \end{align*}
    \label{def:ratio}
\end{definition}

Now we show that the ratio estimator is a special case of the regression estimator. 

\begin{lemma}
    Suppose we have univariate auxiliary information $z_{s,a}$. Under the linear model $\EE_\xi[Y_{\state,a}] = \beta z_{\state,a}$ and $\VV_\xi(Y_{\state,a})=\sigma_{\state,a}^2=z_{\state,a} \sigma^2$ for some $\beta\in\RR$ and $\sigma\in\RR^+$, the regression estimator is equivalent to the ratio estimator. 
    \label{lemma1}
\end{lemma}

\begin{proof}
    First note that the regression estimator has an alternative expression as 
    \begin{align*}
        \hat t_\text{Reg} 
        &= \sum_{(\state,a)\in\calU} z_{\state,a} \hat \beta + \sum_{(\state,a)\in\sample} \frac{y_{\state,a}-z_{\state,a}\hat\beta}{np_{\state,a}} \\
        &= \sum_{(\state,a)\in\sample} \frac{y_{\state,a}}{np_{\state,a}} + \left(\sum_{(\state,a)\in\calU} z_{\state,a} - \sum_{(\state,a)\in\sample} \frac{z_{\state,a}}{np_{\state,a}}\right) \hat\beta \\ 
        &= \sum_{(\state,a)\in\sample} \frac{y_{\state,a}}{np_{\state,a}} + \left(\sum_{(\state,a)\in\calU} z_{\state,a} - \sum_{(\state,a)\in\sample} \frac{z_{\state,a}}{np_{\state,a}}\right) \left(\sum_{(\state,a)\in\sample} \frac{z_{\state,a}z_{\state,a}}{np_{\state,a}\sigma_{\state,a}^2}\right)^\inv \left(\sum_{(\state,a)\in\sample} \frac{z_{\state,a}y_{\state,a}}{np_{\state,a}\sigma_{\state,a}^2}\right) 
        \\
        \\
        &= \sum_{(\state,a)\in\sample} \frac{y_{\state,a}}{np_{\state,a}} \underbrace{\left[1 +  \left(\sum_{(\state,a)\in\calU} z_{\state,a} - \sum_{(\state,a)\in\sample} \frac{z_{\state,a}}{np_{\state,a}}\right) \left(\sum_{(\state,a)\in\sample} \frac{z_{\state,a}z_{\state,a}}{np_{\state,a}\sigma_{\state,a}^2}\right)^\inv  \frac{z_{\state,a}}{\sigma_{\state,a}^2} \right]}_{g_{s.a}} 
        \\
        &= \sum_{(\state,a)\in\sample} \frac{g_{\state,a} y_{\state,a}}{np_{\state,a}}
    \end{align*}
    where $g_{\state,a}$ can be viewed as the weight for each unit in the sample. 
    
    Under the model $\sigma_{\state,a}^2=z_{\state,a} \sigma^2$, for each $(s',a')\in\sample$, we have
    \begin{align*}
        g_{\state',a'}
        &= 1 +  \left(\sum_{(\state,a)\in\calU} z_{\state,a} - \sum_{(\state,a)\in \sample} \frac{z_{\state,a}}{np_{\state,a}}\right) \left(\sum_{(\state,a)\in\sample} \frac{z_{\state,a}z_{\state,a}}{np_{\state,a}\sigma^2 z_{\state,a}}\right)^\inv \left( \frac{z_{\state',a'}}{\sigma^2 z_{\state',a'}}\right) \\
        &= 1 +  \left(\sum_{(\state,a)\in\calU} z_{\state,a} - \sum_{(\state,a)\in\sample} \frac{z_{\state,a}}{np_{\state,a}}\right) \left(\sum_{(\state,a)\in\sample} \frac{z_{\state,a}}{np_{\state,a}}\right)^\inv \\
        &= 1 + \left(\sum_{(\state,a)\in\calU} z_{\state,a}\right) \left(\sum_{(\state,a)\in\sample} \frac{z_{\state,a}}{np_{\state,a}}\right)^\inv - 1 \\
        &= t_z / \left(\sum_{(\state,a)\in\sample} \frac{z_{\state,a}}{np_{\state,a}}\right)
    \end{align*}
    The second equality follows by cancelling out the right most term with the $\sigma^2$ in the inverse bracket. 
    Note that the weight is the same for each unit in the sample. Plugging into the previous equation, we have 
    \begin{align*}
        \hat t_\text{Reg} 
        &= t_z
        \frac{\sum_{(\state,a)\in\sample} y_{\state,a}/np_{\state,a}}{\sum_{(\state,a)\in\sample} z_{\state,a}/np_{\state,a}} = \hat t_\text{Ratio}.
    \end{align*}
\end{proof}

\begin{proof}[Proof of Theorem 1]
    We first show that the WIS estimator belongs to a class of estimators called the ratio estimator in survey sampling in Definition \ref{def:ratio}. 
    Suppose the auxiliary variable $z_{\state,a} = P(\state) \pi(a|\state)$, and we know $t_z = \sum_{(\state,a)\in\calU} P(\state) \pi(a|\state)=1$. Then, the ratio estimator is 
    \begin{align*}
        \hat t_\text{Ratio}
        &= t_z \frac{\hat t_\text{HH}}{\hat t_z}
        = \left(\sum_{(\state,a)\in\sample} \frac{y_{\state,a}}{n p_{\state,a}}\right) \left(\sum_{(\state,a)\in\sample} \frac{z_{\state,a}}{n p_{\state,a}}\right)^\inv  
        = \sum_{(\state,a)\in\sample} \frac{\pi(a|\state)/\pi_b(a|s)}{\sum_{(s',a')\in\sample} \pi(a'|s')/\pi_b(a'|s')} r(\state,a)
    \end{align*}
    which is the WIS estimator in the OPE literature.
    
    Then, we prove a more general statement that a ratio estimator with univariate auxiliary information $z_{s,a}$ is a special case of the regression estimator  under the linear model $\EE_\xi[Y_{\state,a}] = \beta z_{\state,a}$ and $\VV_\xi(Y_{\state,a})=\sigma_{\state,a}^2=z_{\state,a} \sigma^2$ for some $\beta\in\RR$ and $\sigma\in\RR^+$ in Lemma \ref{lemma1}.
\end{proof}

\subsection{Proof of Theorem \ref{thm2}}
\label{appendix:proof2}

\begin{lemma}[Variance of the HH estimator for multinomial design]
    \label{lemma2}
    Let $z$ be a mapping from $\cS\times\cA$ to $[a, b]$ for two constants $a<b$, and $\hat t_\text{HH}$ be the HH estimator for the variable $y_{s,a}=P(s)\pi(a|s)z(s,a)$. With multinomial design $n$ and $p_{s,a}=P(s)\pi_b(a|s)$, the variance is given by
    \begin{align*}
        \VV(\hat t_\text{HH}) 
        &= \frac{1}{n}\left(\sum_{(\state,a)\in U} \frac{y_{\state,a}^2}{p_{\state,a}} - t_{y}^2\right) 
        = \frac{1}{n} \left(\sum_{(\state,a)\in U} P(\state)\pi(a|\state) \frac{\pi(a|\state)}{\pi_b(a|s)}z(\state,a)^2 - t_{y}^2\right).
    \end{align*}
\end{lemma}
\begin{proof}
    Recall the HH estimator is
    \begin{align*}
        \hat t_\text{HH} = \sum_{\sample}  \frac{y_{s,a}}{np_{s,a}} = \sum_{\calU}  \frac{I_{s,a}y_{s,a}}{np_{s,a}}.
    \end{align*}
    where $I_{s,a}$ is the $(s,a)$-th element of the design vector $\Ivec$.
    The variance is 
    \begin{align*}
        \VV(\hat t_\text{HH}) 
        &= \VV\left(\sum_{\calU} I_{s,a} \frac{y_{s,a}}{np_{s,a}}\right) 
        = \sum_{(s,a)\in\calU}\VV(I_{s,a}) \left(\frac{y_{s,a}}{np_{s,a}}\right)^2 + \sum_{(s,a)\neq(s',a')} \Cov(I_{s,a},I_{s',a'}) \left(\frac{y_{s,a}}{np_{s,a}}\right)\left(\frac{y_{s',a'}}{np_{s',a'}}\right).
    \end{align*}
    We know $\VV(I_{s,a}) = np_{s,a}(1-p_{s.a})$ and $\Cov(I_{s,a},I_{s',a'}) = -np_{s.a}p_{s',a'}$ from the properties of the multinomial distribution, hence, after some calculation, we have $\VV(\hat t_\text{HH}) = \frac{1}{n}\left(\sum_{(s,a)\in\calU} \frac{y_{s,a}^2}{p_{s,a}} - t_y^2\right)$. 
    The proof is completed by plugging in the value of $y_{s,a}$ and $p_{s,a}$.
\end{proof}

\begin{proof}[Proof of Theorem 2]
    Note that we assume the first term in Eq \eqref{eq:betahat} is invertible. We can write the estimator as  $\hat t_\text{Reg} = \hat t_y + (t_z - \hat t_z) \hat A^\inv \hat C$ where $\hat t_y$ is the HH estimator for $t_y=\sum_\calU P(s)\pi(a|s)r(s,a)$ and $t_z=\sum_\calU \zvec_{s,a}$ and $\hat t_z$ be the HH estimator for $t_z$. $\hat A$ and $\hat C$ denote the first and second matrix is Eq \eqref{eq:betahat} respectively. Moreover, let $t_{zj}$ be the $j$-th element of the vector $t_z$, and $\hat t_{zj}$ be the $j$-th element of the vector $\hat t_{zj}$. 
    
    Let $A = \sum_{\calU} P(s)\pi(a|s)\phi(s,a)\phi(s,a)^\top$, $C=\sum_\calU P(s)\pi(a|s)\phi(s,a)r(s,a)$ and $B = A^\inv C$.
    Using the Taylor linearization technique (see Section 6.6 of \cite{sarndal2003model}), and we can approximate $\hat t_\text{Reg}$ at $\hat t_y = t_y$, $\hat t_1 = t_1$, $\hat t_z = t_z$, $\hat A=A$ and $\hat C=C$:
    \begin{align*}
        \hat t_\text{Reg} 
        &= t_y + 1(\hat t_y - t_y) - \sum_j B_j (\hat t_{z,j} - t_{z,j})  + \sum_{i,j} (t_z - t_z)^\top [-A^\inv E_{ij} A^\inv] C (\hat A_{ij} - A_{ij}) + \\
        & \ \ \ \ \sum_{j} (t_z - \hat t_z)^\top e_{j} (\hat C_j - C_j) + \dots \\
        &= \hat t_y + (t_z - \hat t_z)^\top B + ...
    \end{align*}
    where $E_{ij}$ is a matrix where the $ij$- and $ji$-th elements are one and all other elements are zero, and $e_j$ is a vector where the $j$-th element is one and zero otherwise. 
    
    Since the random variable is bounded, the moments exist. Taking the expectation, we get
    \begin{align}
        \EE[\hat t_\text{Reg}]
        &= \EE[\hat t_y + (t_z - \hat t_z)^\top B] + O(n^{-1}) 
        = t_y + O(n^{-1}).
        \label{eq:approximation}
    \end{align}
    The first equality follows from the remainder terms of the Taylor expansion are the expectations of $(\hat t_y - t_y)^p$ and $(\hat t_{z,j} - t_{z,j})^p$ for $p\geq2$, which is of order $O(1/n)$.
    The second equality follows from $\hat t_z$ is an unbiased estimator for $t_z$. 
    Therefore, $\hat t_\text{Reg}$ is asymptotically unbiased. 
    
    Furthermore, 
    \begin{align*}
        \VV(\hat t_\text{Reg}) 
        &= \EE[(\hat t_\text{Reg} - \EE[\hat t_\text{Reg}])^2] \\ 
        &= \EE[(\hat t_\text{Reg} - t_y + t_y - \EE[\hat t_\text{Reg}])^2] \\
        &= \EE[(\hat t_\text{Reg} - t_y)^2 + (t_y - \EE[\hat t_\text{Reg}])^2 + 2(\hat t_\text{Reg} - t_y)(t_y - \EE[\hat t_\text{Reg}])] \\
        &= \EE[(\hat t_\text{Reg} - t_y)^2 ] + o(n^\inv) \\
        &= \EE[(\underbrace{\hat t_y - \hat t_z B}_{(a)} - t_y + t_xB)^2] + o(n^\inv)
    \end{align*}
    The last equality comes from Eq \eqref{eq:approximation}. Note that (a) is the HH estimator $\hat t_e = \sum_\sample \frac{\pi(a|s)}{n\pi_b(a|s)} (r(s,a) - \phi(s,a)^\top B)$ so the expectation (the first term in the last line) is the variance of (a) which is given by
    $\frac{1}{n} \left(\sum_{(\state,a)\in U} P(\state)\pi(a|\state) \frac{\pi(a|\state)}{\pi_b(a|s)}(r(s,a) - \phi(s,a)^\top B)^2 - t_{e}^2\right)$ with $t_e = \sum_\calU P(s)\pi(a|s) (r(s,a) - \phi(s,a)^\top B)$ by Lemma \ref{lemma2}.
    
    Since the variance converges to zero and the estimator is asymptotically unbiased, we also know $\hat t_\text{Reg}\overset{p}\to t_y$.
\end{proof}

\subsection{Proof of Theorem \ref{thm3}}
\begin{proof}[Proof of consistency]
    Define 
    \begin{align*}
        \hat \V_n (\beta)
        &= \frac{1}{n(n-1)} \left[\sum_{(\state,a)\in\sample} \left(\frac{\pi(a|\state)}{\pi_b(a|\state)} (r(\state,a) - \phi(\state,a)^\top\beta)\right)^2 - n \hat t_e(\beta)^2 \right], \text{ and } \\
        \VV_n(\beta) 
        &= \frac{1}{n} \left(\sum_{(\state,a)\in U} P(\state)\pi(a|\state) \frac{\pi(a|\state)}{\pi_b(a|s)}(r(\state,a) - \phi(\state,a)^\top \beta)^2 - t_e(\beta)^2\right)
    \end{align*}
    where $\hat t_e(\beta) = \sum_{\sample} \frac{\pi(a|\state)}{n\pi_b(a|s)}(r(\state,a) - \phi(\state,a)^\top\beta)$ and $t_e(\beta) = \sum_\calU P(s)\pi(a|s)(r(\state,a) - \phi(\state,a)^\top \beta)$.
    Then it is sufficient to show that
    \begin{align*}
        n|\hat \V_n (\hat \beta_n) - \VV_n(\beta_{WLS})| \overset{p}\to 0.
    \end{align*}
    
    For $\epsilon>0$, by the triangle inequality,
    \begin{align*}
        \Pr(n|\hat \V_n (\hat \beta_n) - \VV_n(\beta_{WLS})| > \epsilon) 
        &\leq \Pr(n|\hat \V_n (\hat \beta_n) - \hat \V_n (\beta_{WLS})|>\epsilon/2) + \Pr(n|\hat \V_n (\beta_{WLS}) - \VV_n(\beta_{WLS})|>\epsilon/2).
    \end{align*}
    For the first term, using the fact that $\hat \beta_n\overset{p}\to\beta_{WLS}$ and the continuous mapping theorem, we get $\hat \V_n (\hat\beta_n) \overset{p}\to \hat \V_n (\beta_{WLS})$, which implies $\lim_{n\to\infty} \Pr(n|\hat \V_n (\hat \beta_n) - \hat \V_n (\beta_{WLS})|>\epsilon/2) = 0$.
    
    Define $e_{s,a}(\beta) = r(\state,a) - \phi(\state,a)^\top\beta$, $t_{we^2}(\beta) = \sum_\calU P(s)\pi(a|s)\frac{\pi(a|\state)}{\pi_b(a|s)}e_{s,a}(\beta)^2$ (the first term of $\VV_n(\beta)$) and $\hat t_{we^2}(\beta) = \sum_{\sample} \left(\frac{\pi(a|\state)}{n\pi_b(a|\state)} e_{s,a}(\beta)\right)^2$ (the first term of $\hat\VV(\beta)$).
    Then, for the second term, we have
    \begin{align*}
        & \Pr(n\left|\hat \V_n (\beta_{WLS}) - \VV_n(\beta)\right|>\epsilon/2) \\
        &= \Pr(\left|\frac{n}{n-1} \hat t_{we^2}(\beta_{WLS}) - \frac{n}{n-1} \hat t_e(\beta_{WLS})^2 - t_{we^2}(\beta_{WLS}) + t_e(\beta_{WLS})^2 \right|>\epsilon/2)\\
        &\leq \Pr(\left|\frac{n}{n-1} \hat t_{we^2}(\beta_{WLS}) - t_{we^2}(\beta_{WLS}) \right|>\epsilon/4) + \Pr( \left|\frac{n}{n-1} \hat t_e(\beta_{WLS})^2  - t_e(\beta_{WLS})^2\right|<\epsilon/4).
    \end{align*}
    Note that $\hat t_{we^2}(\beta_{WLS})$ and $\hat t(\beta_{WLS})^2$ are the HH estimators for $t_{we^2}(\beta_{WLS})$ and $t_e(\beta_{WLS})$ respectively, so they are consistent. As a result, we have
    \begin{align*}
        \lim_{n\to\infty} \Pr(n\left|\hat \V_n (\hat\beta_{WLS}) - \VV_n(\beta_{WLS})\right|>\epsilon/2) = 0, 
    \end{align*}
    which completes the proof.
\end{proof}

\begin{proof}[Proof of asymptotic normality]
    It is known that the HH estimator for with-replacement sampling is asymptotically normal (for example, see Theorem 2 of \cite{felix2003asymptotics} or \cite{mcconville2011improved}), that is,
    \begin{align*}
        \begin{bmatrix}
        \sqrt{n}(\hat t_y - t_y)\\
        \sqrt{n}(\hat t_z - t_z)
        \end{bmatrix} \overset{D}\to \mathcal{N}\left(0, 
        \begin{bmatrix}
        \Sigma^y & \mathbf\Sigma^{yz} \\
        \mathbf\Sigma^{zy} & \mathbf\Sigma^{z}
        \end{bmatrix}\right)
    \end{align*}
    where $\Sigma^y,\mathbf\Sigma^{yz},\mathbf\Sigma^{zy}$ and $\mathbf\Sigma^{z}$ are the limiting covariance matrices. 
    Then we follow the proof idea from Theorem 3.2 of \cite{mcconville2011improved}. 
    Using the Slutsky’s Theorem and the fact that $\hat \beta_n \overset{p}\to \beta_{WLS}$, we have 
    \begin{align*}
        \begin{bmatrix}
        \sqrt{n}(\hat t_y - t_y)\\
        \sqrt{n}(\hat t_z - t_z)\hat \beta_n
        \end{bmatrix} \overset{D}\to \mathcal{N}(0, 
        \begin{bmatrix}
        \Sigma^y & \mathbf\Sigma^{yz} \beta_{WLS} \\
        \beta_{WLS}^\top \mathbf\Sigma^{zy} & \beta_{WLS}^\top \mathbf\Sigma^{z} \beta_{WLS}
        \end{bmatrix}).
    \end{align*}
    Note that $\sqrt{n}(\hat t_\text{Reg}-t_y) = \sqrt{n}(\hat t_y - t_y) - \sqrt{n}(\hat t_z - t_z)\hat \beta_n$. By the Delta method, we have 
    $\sqrt{n}(\hat t_\text{Reg}-t_y) \overset{D}\to \mathcal{N}(0,\sigma^2)$
    where $\sigma^2 = \Sigma^y - \Sigma^{yz} \beta_{WLS} - \beta_{WLS}^\top \Sigma^{zy} + \beta_{WLS}^\top \Sigma^{z} \beta_{WLS}$.
    Note that we can write the variance of $\hat t_y - \hat t_z \beta_{WLS}$ as $\VV(\hat t_y - \hat t_z \beta_{WLS}) = \frac{1}{n}(\Sigma^y - \Sigma^{yz} \beta_{WLS} - \beta_{WLS}^\top \Sigma^{zy} + \beta_{WLS}^\top \Sigma^{z} \beta_{WLS})$, and in the proof for Theorem \ref{thm2}, we show that the asymptotic variance of $\hat t_\text{Reg}$ is $\AV(\hat t_\text{Reg})=\VV(\hat t_y - \hat t_z \beta_{WLS})$.
    Therefore,
    $\AV(\hat t_\text{Reg})=\sigma^2/n$, and 
    $(\hat t_\text{Reg}-t_y)/\sqrt{\AV(\hat t_\text{Reg})} \overset{D}\to \mathcal{N}(0,1)$.
    
    By the consistency of the variance estimator and Slutsky's theorem, we have 
    \begin{align*}
        \frac{\hat t_\text{Reg}-t_y}{\sqrt{\hat \VV(\hat t_\text{Reg})}} 
        = \frac{\hat t_\text{Reg}-t_y}{\sqrt{\AV(\hat t_\text{Reg})}} \frac{\sqrt{\AV(\hat t_\text{Reg})}}{\sqrt{\hat \VV(\hat t_\text{Reg})}}
        \overset{D}\to \mathcal{N}(0,1).
    \end{align*} 
\end{proof}

\section{VARIANCE ESTIMATION}
\label{appendix:regression_variance}
In this section, we provide the variance estimation for all estimators used in our experiments. 

\paragraph{Variance estimation for the IS estimators.}
For the IS estimator from the Monte Carlo literature, we first note that the estimator can be written as $\frac{1}{n}\sum_i W_i R_i$ where $W_i=\frac{\pi(A_i|\stateRV_i)}{\pi_b(A_i|\stateRV_i)}$ and $R_i = r(\stateRV_i,A_i)$.
Then, the variance is given by $\VV(\hat J_\text{IS}(\pi)) = \frac{1}{n} \VV\left(W R\right)$ due to the i.i.d. property, and $\VV\left(W R\right)$ can be estimated by the sample variance. Therefore, we have an unbiased variance estimator
\begin{align*}
    \hat \VV(\hat J_\text{IS}(\pi)) 
    = \frac{1}{n} \left[\frac{1}{n-1}\sum_{i=1}^n \left(W_i R_i - \hat J_\text{IS}(\pi)\right)^2\right].
\end{align*}

For the HH estimator from the survey sampling literature, we can use the Sen-Yates-Grundy variance estimator \citep{sen1953estimate,yates1953selection} for the multinomial design, which is given by
\begin{align*}
    \hat \VV(\hat t_\text{HH}) 
    = \frac{1}{n(n-1)} \left[\sum_{(\state,a)\in\sample} \left(\frac{\pi(a|\state)}{\pi_b(a|x)} r(\state,a)\right)^2 - n \hat t_\text{HH}^2 \right].
\end{align*}
The variance estimator $\hat \VV(\hat t_\text{HH})$ is an unbiased estimator for the true variance $\VV(\hat t_\text{HH})$. Interestingly, it also has the same form as the variance estimator for the IS estimator.


\paragraph{Variance estimation for the WIS estimator.}
Using the Taylor linearization technique, the ratio estimator is approximately by $\hat t_\text{Ratio} = t_x \hat R t_x \approx R + (\hat t_y - R \hat t_x)$.
Define $u_{s,a} = y_{s,a}  - Rx_{s,a} $, $t_u = \sum_{i\in\calU} u_{s,a} $ and $\hat t_u = \frac{1}{n}\sum_{\sample} \frac{u_{s,a} }{p_{s,a} }$, then we have an approximate variance $\AV(\hat t_\text{Ratio}) = \V(\hat t_u)$. 
Based on the approximation, the estimated variance is given by
\begin{align*}
    \hat \VV(\hat t_\text{WIS}) 
    &= \frac{1}{n(n-1)} \left[\sum_{(\state,a)\in\sample} \left(\frac{\pi(a|\state)}{\pi_b(a|x)} (r(\state,a) - \hat t_\text{WIS})\right)^2 - n \left(\frac{1}{n} \sum_{(\state,a)\in\sample} \frac{\pi(a|\state)}{\pi_b(a|s)}(r(\state,a) - \hat t_\text{WIS})\right)^2\right].
\end{align*}

\paragraph{Variance estimation for the difference and DR estimator.}
Since the first term of the difference estimator is fixed, the variance of the difference estimator equals to the variance of the HH estimator $\hat t_{\Delta} = \frac{1}{n}\sum_{(\state,a) \in\sample} \frac{\pi(a|\state)}{ \pi_b(a|s)}\Delta(\state,a)$ where $\Delta(\state,a) = r(\state,a) - \hat r(\state,a)$. Then the variance estimator is
\begin{align*}
    \hat \VV(\hat t_\text{Diff}) 
    = \frac{1}{n(n-1)} \left[\sum_{(\state,a)\in \sample} \left(\frac{\pi(a|\state)}{\pi_b(a|x)} \Delta(\state,a)\right)^2 - n \hat t_{\Delta}^2 \right].
\end{align*}

For the DR estimator where the first term is also estimated from $\sample$, first note that the DR estimator can be written as 
\begin{align*}
    \hat t_\text{DR}
    &= \frac{1}{n} \sum_{(\state,a) \in \sample} \frac{(\pi(a|x) (r(\state,a) - \hat r(\state,a))+ \pi_b(a|s) \hat r_\pi(\state))}{\pi_b(a|s)}
\end{align*}
which is the HH estimator for $t = \sum_{\calU} P(\state) (\pi(a|\state)(r(\state,a) - \hat r(\state,a))+ \pi_b(a|s) \hat r_\pi(\state))$. Therefore, we have the variance estimator
\begin{align*}
    \hat \VV(\hat t_\text{DR}) 
    = \frac{1}{n(n-1)} \left[\sum_{(\state,a)\in\sample} \left(\frac{(\pi(a|x) (r(\state,a) - \hat r(\state,a))+ \pi_b(a|s) \hat r_\pi(\state))}{\pi_b(a|s)}\right)^2 - n \hat t_\text{DR}^2 \right].
\end{align*}

\paragraph{Variance estimation for the regression-assisted estimator.}


We briefly describe the weighted residual
technique from \cite{sarndal1989weighted}. Using the definition of $g_{\state,a}=1+(t_x-\hat t_x)(\sum_{i\in\sample} \frac{\pi(a|\state)\phi(\state,a) \phi(\state,a)^\top}{n \pi_b(a|s)})^\inv \phi(\state,a)$, the regression estimator can be written as 
\begin{align*}
    \hat t_\text{Reg} = \sum_{(\state,a)\in\calU} P(\state)\pi(a|\state)\phi(\state,a)^\top\beta_{WLS} + \sum_{(\state,a)\in\sample} \frac{\pi(a|\state)}{n \pi_b(a|s)}g_{\state,a}(r(\state,a) - \phi(\state,a)^\top\beta_{WLS}).
\end{align*}

It follows that 
\begin{align*}
    \VV(\hat t_\text{Reg}) = \VV\left(\sum_{(\state,a)\in\sample} \frac{ \pi(a|\state)}{n \pi_b(a|s)}g_{\state,a}(r(\state,a) - \phi(\state,a)^\top\beta_{WLS})\right)
\end{align*}
which is the variance of the HH estimator $\hat t_e = \sum_\sample \frac{\pi(a|\state)}{\pi_b(a|s)}g_{sa}(r(\state,a) - \phi(\state,a)^\top\beta)$ for $t_e = \sum_{\calU} P(\state)\pi(a|\state)g_{\state,a}(r(\state,a) - \phi(\state,a)^\top\beta)$. 
Ignoring the fact that the weight $g_{\state,a}$ is sample dependent and replacing $\beta_{WLS}$ with $\hat \beta$, we have the $g$-weighted variance estimator
\begin{align*}
    \hat \VV(\hat t_\text{Reg})
    = \frac{1}{n(n-1)} \left[\sum_{(\state,a)\in\sample} \left(\frac{\pi(a|\state)}{\pi_b(a|x)} g_{sa} (r(\state,a) - \phi_{\state,a}\hat\beta)\right)^2 - n \hat t_e^2 \right].
\end{align*}


\section{EXTENSION TO RL}
\label{appendix:rl}
In the main paper, we discuss OPE for RL by treating each trajectory as one unit in a population containing all possible trajectories. That is, let $\calU = (\stateSet\times\cA)^H$ be the population containing all trajectories and $y_{\tau} = \bbP^\pi_M(\tau)R(\tau)$ be the study variable.
We use multinomial design with $p_{\tau} = \bbP^{\pi_b}_M(\tau)$ to obtain a sample of trajectories $\sample$. 
The resulting HH estimator has the same form as the trajectory-wise IS estimator \citep{sutton1998introduction}.
, that is,
\begin{align*}
    \hat t_\text{IS}
    = \frac{1}{n} \sum_{\tau \in \sample} \frac{\bbP^{\pi}_M(\tau)}{\bbP^{\pi_b}_M(\tau)} R(\tau)
    = \frac{1}{n} \sum_{\tau \in \sample} \frac{\pi(a_0|s_0)\dots\pi(a_{H-1}|s_{H-1})}{\pi_b(a_0|s_0)\dots\pi_b(a_{H-1}|s_{H-1})} R(\tau). 
\end{align*}

However, there are many other ways to view OPE for RL in the survey sampling framework. One way is to consider estimating the expected reward at each horizon separately. In this case, for $h=0,\dots,H-1$, let $\calU_h = (\stateSet\times\cA)^{h+1}$ and our goal is to estimate 
\[J_h(\pi)=\sum_{\tau \in \calU_h}\nu(s_0)\pi(a_0|s_0)P(s_1|s_0,a_0) \dots \pi(a_{h}|s_{h}) r(s_{h},a_{h})\]
which is the expected reward at horizon $h$ under policy $\pi$. It is easy to see that $J(\pi)=\sum_{h=0}^{H-1}J_h(\pi)$. Therefore, this can be viewed as stratified sampling where stratum $h$ contain all trajectories of horizon $h$. Using the IS estimator to estimate $J_h(\pi)$ for each horizon, we get the per-decision IS estimator \citep{precup2000eligibility}. That is,
\begin{align}
    \hat t_\text{PDIS}
    = \sum_{h=0}^{H-1} \hat J_h(\pi)
    = \sum_{h=0}^{H-1} \frac{1}{n} \sum_{\tau \in\sample} \frac{\pi(a_0|s_0)\dots\pi(a_{h}|s_{h})}{\pi_b(a_0|s_0)\dots\pi_b(a_{h}|s_{h})} r(s_h,a_h). 
    \label{eq:PDIS}
\end{align}
The sampling at each horizon might depend on the sampling at the previous horizon. In that case, we can't easily get the variance or an variance estimator since $V(\hat J(\pi))\neq\sum_{h=0}^{H-1}V(\hat J_h(\pi))$. However, we can still compute the variance and an variance estimator via a recursive form \citep{jiang2016doubly}. 

Another way is to consider sampling transitions instead of episode. Let $\calU = \stateSet\times\cA$ and for $(\state,a)\in\calU$, $y_{\state,a} = d^\pi(\state,a) r(\state,a)$ where $d^\pi(\state,a)=(\sum_{h=0}^{H-1} \bbP^\pi_M(S_h=\state,a_h=a))/H$. We use multinomial design with $p_{\state,a} = \mu(\state,a)$ where $\mu$ is a data distribution that we can use to sample data. 
However, for this formulation, $d^\pi(\state,a)$ is usually unknown, and there are existing work on estimating the density ratio \citep{liu2018breaking}. 


\section{EXPERIMENT DETAILS}
\label{appendix:exp_details}

The goal of the experiment design is to model the recommendation system where reward associated with each user-item pair changes over time. For the Youtube dataset, we generate a sequence of reward functions based on the non-stationary recommendation environment used in \cite{chandak2020optimizing}. For each positive context-action pair in the original classification dataset, the reward follows a sine wave with noises. More precisely, $r_k(s,a)=0.5 + \text{amplitude}_{s,a} * \sin(k* \text{frequency}_{s,a}) + 0.01 \varepsilon$ where $\varepsilon\sim\text{Unif}([0,1])$. For each interval, we also randomly sample some context-action pairs and set their rewards to positive random values to increase the noise. 

To obtain a target policy for the Youtube and MovieLens dataset, we fist train a classifier on a small subset of the original multi-label classification dataset. Then we apply the softmax function on the outputs of the trained classifiers to obtain a probability distribution over actions for each context. The conditional distribution is used as the target policy.  

Similarly for the RL environment, the reward follows $r_k(s,a)= \mu_{s,a} + 0.25 * \sin(k* \text{frequency}_{s,a}) + 0.01 * \varepsilon$ where $\varepsilon\sim\text{Unif}([0,1])$. To obtain a target policy, we fist train a Q-learning agent on the underlying environment for $1000$ episodes and then apply the softmax function on the Q-value as the target policy. 

The summary statistics of the dataset are: 
\begin{table}[h!]
    \centering
    \begin{tabular}{c|c|c}
        & $|\stateSet|$ & $|\cA|$ \\
        \hline
        Youtube & 31703 & 47 \\
        MovieLens & 1923 & 19 \\
    \end{tabular}
\end{table}

We provide a pseudocode for our experiment procedure in Algorithm \ref{algo:exp_code}.  
\begin{algorithm}[h]
    \caption{Non-stationary OPE experiment with regression-assisted DR estimator}
    \label{algo:exp_code}
    \begin{algorithmic}
        \STATE Input: a non-stationary environment $M$, a target policy $\pi$, a behavior policy $\pi_b$, window size $B$, a prediction subroutine $Pred(X,Y,x_{test})$ using linear regression with basis function $\psi$
        \FOR {$k=0,\dots,K$}
            \STATE Collect a dataset $D_k=\{(s_i,a_i,r_k(s_i,a_i),\pi_b(a_i|s_i))\}_{i=1}^n$ from $M$
            \IF {$k>0$}
                \STATE \emph{\# Estimate the current value}
                \STATE Build a reward prediction $\hat r_{k}$ from the past data $D_{k-B},\dots,D_{k-1}$ 
                \STATE Compute $\hat\beta_k$ from $D_k$ using Eq \eqref{eq:betahat} with $\phi(s,a)=(1,\hat r_{k}(s,a))$
                \STATE Compute $\hat t_{Reg,k}$ from $D_k$ using Eq \eqref{eq:reg}
                \STATE \emph{\# Predict the future value}
                \STATE $\hat t_{Pred,k+1} = Pred(X=[\psi(1),\dots,\psi(k)], Y=[\hat t_{Reg,1},\dots,\hat t_{Reg,k}], x_{test}=\psi(k+1))$
            \ENDIF
        \ENDFOR
        \STATE Compute the true value $J_{1}(\pi),\dots,J_{k}(\pi)$
        \STATE Output: 
        \STATE 
        $\mathrm{RMSE}_{Reg}=\sqrt{\frac{1}{K}\sum_{k=1}^K (\hat t_{Reg,k} - J_k(\pi))^2}$ \emph{\# Error for estimating the current value}
        \STATE  $\mathrm{RMSE}_{Pred}=\sqrt{\frac{1}{K}\sum_{k=2}^K (\hat t_{Pred,k} - J_k(\pi))^2}$ \emph{\# Error for predict the future value}
    \end{algorithmic}
\end{algorithm}


\section{ADDITIONAL EXPERIMENTS}
\label{appendix:more_exp}
We provide more experiment results in this section. 

\paragraph{Estimating the population total of the proxy values.}
For the experiments in the main paper, we use the population total of the proxy values for the DM, Diff and Reg estimator. In this experiment, we test the regression-assisted estimator when the population total of the proxy values are being estimated, that results in the estimator using Eq \eqref{eq:DR}, which we call RegDR, and the estimator with an independent survey $D'$, described in the last paragraph of Section \ref{sec:diff}, which we call RegDR2. 

In Figure \ref{fig:total}, we found that RegDR2 has a similar RMSE compared to Reg, and both are slightly better than RegDR. This suggests that using an independent survey has potential to improve the standard DR estimator in the non-stationary setting. 
Moreover, even the population total needs to be estimated, the regression-assisted estimator can still perform well using RegDR2.

\begin{figure}[t]
    \centering
    \includegraphics[width=0.9\textwidth]{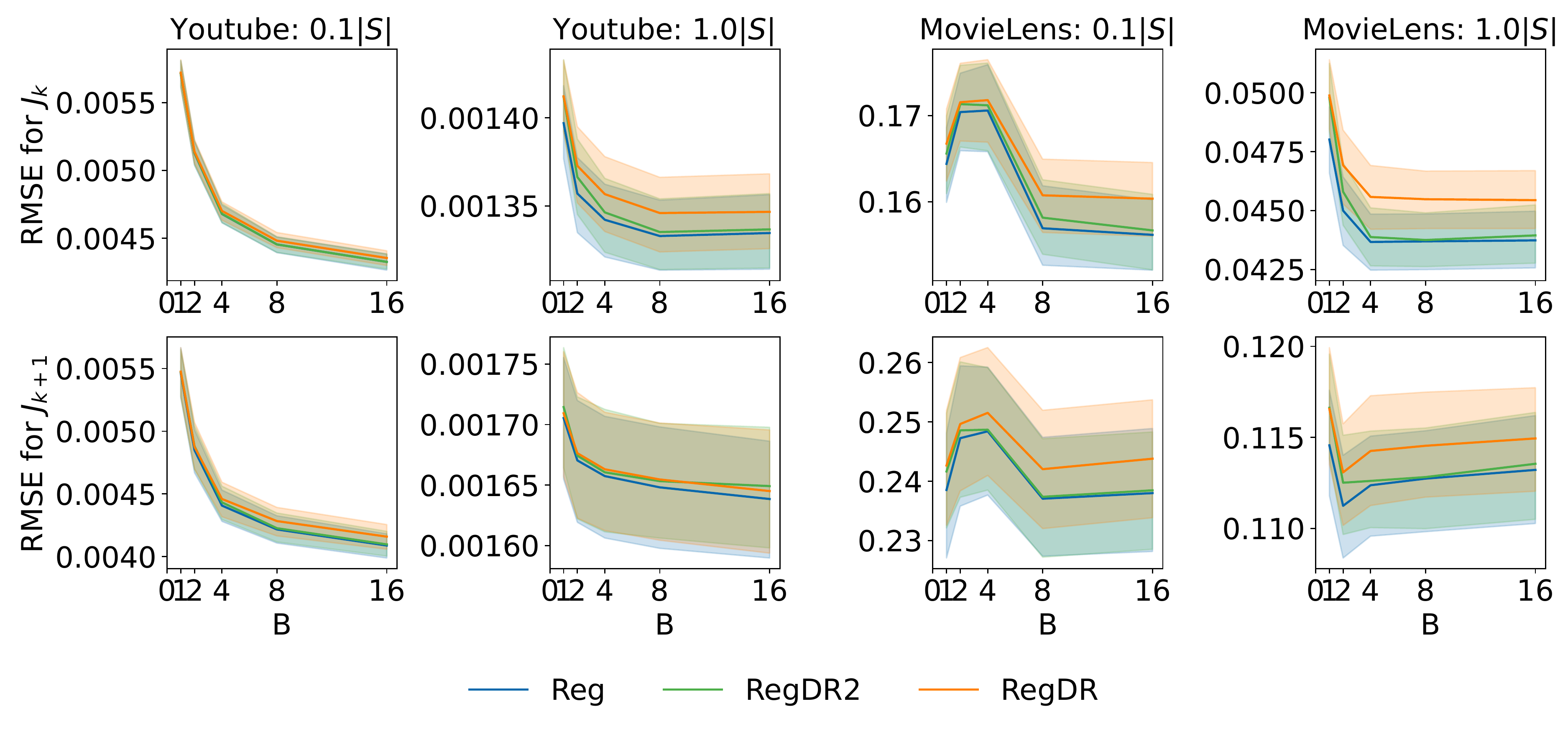}
    \caption{Comparison when the population total of the proxy value is estimated. \textbf{Top: estimating $J_k(\pi)$. Bottom: predicting $J_{k+1}(\pi)$.}}
    \label{fig:total}
\end{figure}



\paragraph{Comparison to PDIS.}
For the RL experiment, we compare Reg to PDIS in Eq \eqref{eq:PDIS}, which is expected to improve over the standard trajectory-wise IS estimator. The result in Figure \ref{fig:PDIS} shows that Reg still outperforms PDIS. 

\begin{figure}[t]
    \centering
    \includegraphics[width=0.5\textwidth]{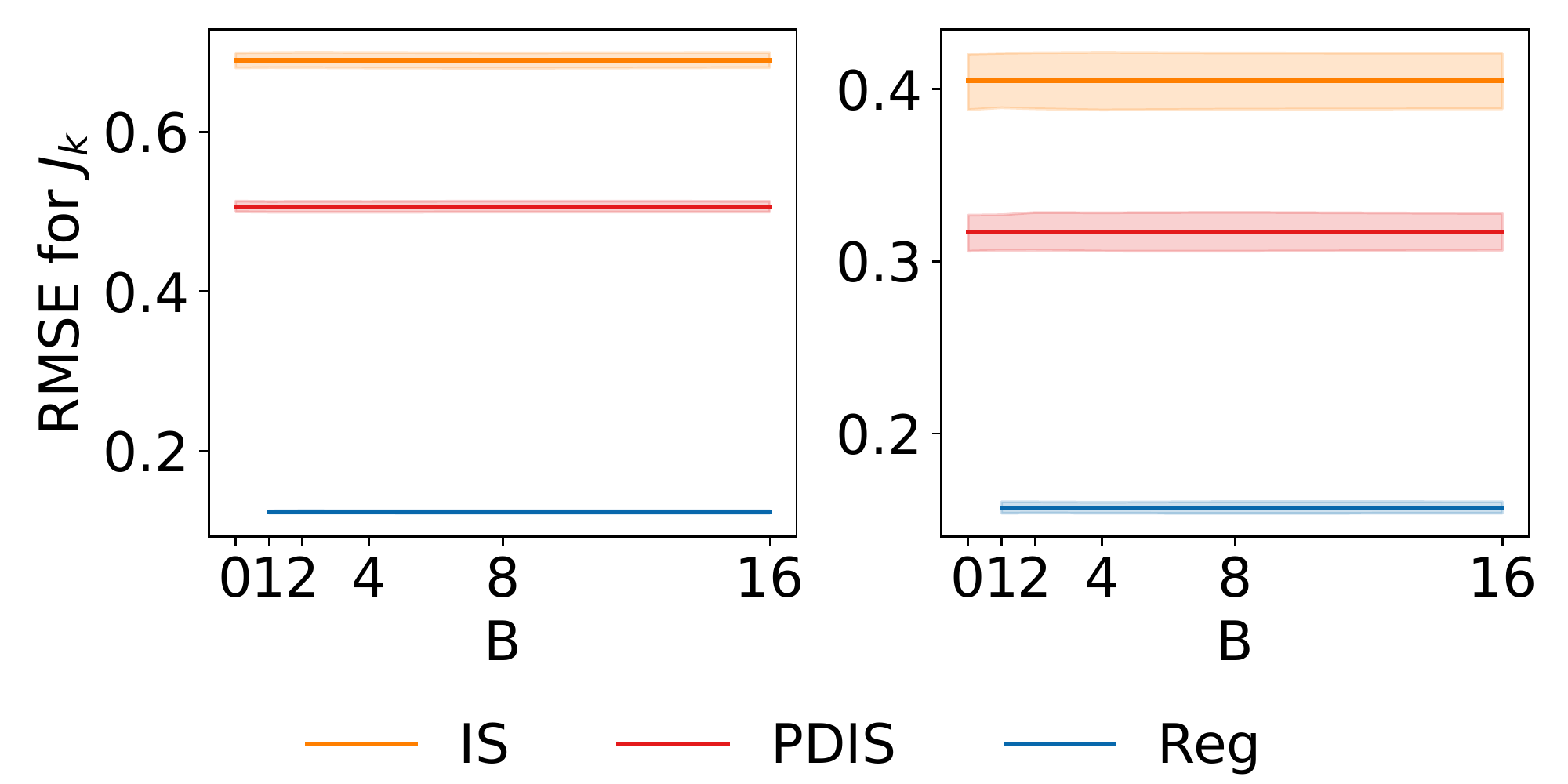}
    \caption{Comparison to PDIS in the simulated RL environment. \textbf{Left: estimating $J_k(\pi)$. Right: predicting $J_{k+1}(\pi)$.} Note that PDIS and IS are horizontal lines since they do not depend on $B$. }
    \label{fig:PDIS}
\end{figure}

\end{document}